\documentclass{article}

\usepackage{microtype}
\usepackage{graphicx}
\usepackage{booktabs} %

\usepackage{hyperref}

\usepackage{amsmath,amsfonts,bm}

\def\eqref#1{equation~\ref{#1}}

\def\1{\bm{1}}

\def\vb{{\bm{b}}}

\def\vg{{\bm{g}}}

\def\vp{{\bm{p}}}
\def\vq{{\bm{q}}}

\def\vu{{\bm{u}}}
\def\vv{{\bm{v}}}

\def\vx{{\bm{x}}}

\def\vz{{\bm{z}}}

\def\mA{{\bm{A}}}
\def\mB{{\bm{B}}}

\def\mI{{\bm{I}}}

\def\mM{{\bm{M}}}

\def\mW{{\bm{W}}}

\DeclareMathAlphabet{\mathsfit}{\encodingdefault}{\sfdefault}{m}{sl}
\SetMathAlphabet{\mathsfit}{bold}{\encodingdefault}{\sfdefault}{bx}{n}

\newcommand{\E}{\mathbb{E}}

\newcommand{\vectorize}{\mathrm{vec}}

\DeclareMathOperator*{\argmin}{arg\,min}

\def\comma{{ \text{ ,} }}
\def\period{{ \text{ .} }}

\newcommand{\norm}[1]{\left\lVert#1\right\rVert}

\def\etaCL{{\eta^{\text{C}}_{t,n}}}
\def\etaOL{{\eta^{\text{O}}_{t,n}}}
\def\etaCCG{{\eta^{\text{C-CG}}_{t,n}}}
\def\etaOCG{{\eta^{\text{O-CG}}_{t,n}}}

\def\Hut{{H_\theta^t}}

\DeclareMathOperator{\diag}{diag}

\def\ktn{{\mathbf{k}_{t,n}}}
\def\Ktn{{\mathbf{K}_{t,n}}}

\def\Hutn{{{H}_{\theta}^{t,n}}}

\def\Puut{{P_{\vu\vu}^t}}
\def\Pvvt{{P_{\vv\vv}^t}}
\def\PuuInvt{{P_{\vu\vu}^{t\text{ }\text{ }\Inv}}}
\def\PvvInvt{{P_{\vv\vv}^{t\text{ }\text{ }\Inv}}}
\def\PuuCInvt{{\widetilde{P}_{\vu\vu}^{t\text{ }\text{ }\Inv}}}
\def\PvvCInvt{{\widetilde{P}_{\vv\vv}^{t\text{ }\text{ }\Inv}}}
\def\Put{{P_{\vu}^t}}
\def\Puxt{{P_{\vu\vx}^t}}
\def\Pvt{{P_{\vv}^t}}
\def\Pvxt{{P_{\vv\vx}^t}}
\def\Puvt{{P_{\vu\vv}^t}}
\def\Pvut{{P_{\vv\vu}^t}}
\def\PuuCt{{\widetilde{P}_{\vu\vu}^t}}
\def\PvvCt{{\widetilde{P}_{\vv\vv}^t}}
\def\Pxut{{P_{\vx\vu}^t}}
\def\Pxvt{{P_{\vx\vv}^t}}

\def\kt{{\mathbf{k}_t}}
\def\Kt{{\mathbf{K}_t}}

\def\It{{\mathbf{I}_t}}
\def\Lt{{\mathbf{L}_t}}

\def\kut{{\widetilde{\mathbf{k}}_t}}
\def\Kut{{\widetilde{\mathbf{K}}_t}}
\def\KutT{{\widetilde{\mathbf{K}}_t^\transpose}}

\def\Ivt{{\widetilde{\mathbf{I}}_t}}
\def\Lvt{{\widetilde{\mathbf{L}}_t}}
\def\LvtT{{\widetilde{\mathbf{L}}_t^\transpose}}

\def\Axx{{\E[\vz_{1}\vz_{1}^\transpose]}}
\def\Axy{{\E[\vz_{1}\vz_{2}^\transpose]}}
\def\Ayy{{\E[\vz_{2}\vz_{2}^\transpose]}}

\def\Bxx{{\E[\vg_{1}\vg_{1}^\transpose]}}
\def\Bxy{{\E[\vg_{1}\vg_{2}^\transpose]}}
\def\Byy{{\E[\vg_{2}\vg_{2}^\transpose]}}

\def\Auu{{A_{\vu\vu}}}
\def\Auv{{A_{\vu\vv}}}
\def\Avu{{A_{\vv\vu}}}
\def\Avv{{A_{\vv\vv}}}
\def\Buu{{B_{\vu\vu}}}
\def\Buv{{B_{\vu\vv}}}
\def\Bvu{{B_{\vv\vu}}}
\def\Bvv{{B_{\vv\vv}}}
\def\AuvT{{A^\transpose_{\vu\vv}}}
\def\BuvT{{B^\transpose_{\vu\vv}}}
\def\AvuT{{A^\transpose_{\vv\vu}}}
\def\BvuT{{B^\transpose_{\vv\vu}}}
\def\AuuInv{{A^{\Inv}_{\vu\vu}}}

\def\AvvInv{{A^{\Inv}_{\vv\vv}}}
\def\BuuInv{{B^{\Inv}_{\vu\vu}}}
\def\BvvInv{{B^{\Inv}_{\vv\vv}}}
\def\AuuC{{\widetilde{A}_{\vu\vu}}}
\def\AvvC{{\widetilde{A}_{\vv\vv}}}
\def\BuuC{{\widetilde{B}_{\vu\vu}}}
\def\BvvC{{\widetilde{B}_{\vv\vv}}}

\def\AuuCInv{{\widetilde{A}^{\Inv}_{\vu\vu}}}
\def\BuuCInv{{\widetilde{B}^{\Inv}_{\vu\vu}}}
\def\AvvCInv{{\widetilde{A}^{\Inv}_{\vv\vv}}}
\def\BvvCInv{{\widetilde{B}^{\Inv}_{\vv\vv}}}
\def\AvvInvT{{A^{-\transpose}_{\vv\vv}}}
\def\AuuCInvT{{\widetilde{A}^{-\transpose}_{\vu\vu}}}

\def\dvx{{\delta\vx}}

\def\dvv{{\delta\vv}}

\def\Inv{{\dagger}}

\def\dth{{\delta\theta}}

\def\Qxt{{Q^{t,n}_{\vx}}}
\def\Qut{{Q^{t,n}_{\theta}}}
\def\Qxxt{{Q^{t,n}_{\vx \vx}}}
\def\Quxt{{Q^{t,n}_{\theta \vx}}}

\def\Quut{{Q^{t,n}_{\theta \theta}}}

\def\Quutn{{ Q^{t,n}_{\theta \theta} }}
\def\Qutn{{ Q^{t,n}_{\theta} }}
\def\Quxtn{{ Q^{t,n}_{\theta\vx} }}

\def\FutT{{{{F}^t_{\theta}}^\transpose}}
\def\FxtT{{{{F}^t_{\vx}}^\transpose}}
\def\Fxt{{{F}^t_{\vx}}}
\def\Fut{{{F}^t_{\theta}}}
\def\lutn{{{\ell}^{t,n}_{\theta}}}
\def\luutn{{{\ell}^{t,n}_{\theta\theta}}}
\def\Vxtn{{V^{t+1,n}_{\vx}}}
\def\Vxxtn{{V^{t+1,n}_{\vx\vx}}}

\def\Gu{{\mathbf{G}^t_\vu}}

\def\transpose{{\mathsf{T}}}

\newcommand{\eg}{{\ignorespaces\emph{e.g.}}{ }}
\newcommand{\ie}{{\ignorespaces\emph{i.e.}}{ }}
\newcommand{\wlg}{{\ignorespaces\emph{w.l.o.g.}}{ }}

\newcommand{\resp}{{\ignorespaces\emph{resp.}}{ }}

\usepackage{mathtools}
\usepackage{tabularx}
\usepackage{framed}

\usepackage{amssymb}
\usepackage{amsthm}
\usepackage{multirow}

\newtheorem{theorem}{Theorem}
\newtheorem{lemma}[theorem]{Lemma}
\newtheorem{proposition}[theorem]{Proposition}

\newtheorem{definition}[theorem]{Definition}

\newcommand{\markred}[1]{{\color{red} #1}}
\newcommand\numberthis{\addtocounter{equation}{1}\tag{\theequation}}
\usepackage{cancel}
\usepackage{lipsum}

\usepackage{capt-of}
\usepackage{wrapfig}

\usepackage{xcolor}
\usepackage{color,soul}
\colorlet{color1}{green!50!black}
\colorlet{color2}{orange!95!black}
\colorlet{color3}{red!80!black}
\colorlet{color4}{red!65!black}
\colorlet{color5}{blue!75!green}
\newcommand{\markgreen}[1]{{\ignorespaces\color{color1} #1}}
\newcommand{\markblue}[1]{{\ignorespaces\color{color5} #1}}

\usepackage{footnote}

\let\oldsqrt\sqrt
\def\sqrt{\mathpalette\DHLhksqrt}
\def\DHLhksqrt#1#2{\setbox0=\hbox{$#1\oldsqrt{#2\,}$}\dimen0=\ht0
\advance\dimen0-0.2\ht0
\setbox2=\hbox{\vrule height\ht0 depth -\dimen0}%
{\box0\lower0.4pt\box2}}

\newcommand{\specialcell}[2][c]{%
  \begin{tabular}[#1]{@{}c@{}}#2\end{tabular}}

\newcommand{\specialcelll}[2][l]{%
  \begin{tabular}[#1]{@{}l@{}}#2\end{tabular}}

\usepackage{enumitem}
\usepackage{subfig}
\usepackage{arydshln}

\usepackage{pgf}

\usepackage[accepted]{icml2021}

\icmltitlerunning{DGNOpt: Dynamic Game Theoretic Neural Optimizer}

\begin{document}

\twocolumn[
\icmltitle{Dynamic Game Theoretic Neural Optimizer}

\begin{icmlauthorlist}
\icmlauthor{Guan-Horng Liu}{ml,ae}
\icmlauthor{Tianrong Chen}{ece}
\icmlauthor{Evangelos A. Theodorou}{ml,ae}
\end{icmlauthorlist}

\icmlaffiliation{ml}{Center for Machine Learning}
\icmlaffiliation{ae}{School of Aerospace Engineering}
\icmlaffiliation{ece}{School of Electrical and Computer Engineering, Georgia Institute of Technology, USA}

\icmlcorrespondingauthor{Guan-Horng Liu}{ghliu@gatech.edu}

\icmlkeywords{DNN training, Dynamic game, Game-theoretic optimizer}

\vskip 0.3in
]

\printAffiliationsAndNotice{}  %

\begin{abstract}
The connection between training deep neural networks (DNNs) and optimal control theory (OCT) has attracted considerable attention as a principled tool of algorithmic design.
Despite few attempts being made,
they have been limited to architectures where the layer propagation resembles a {Markovian dynamical system}.
This casts doubts on their flexibility
to modern networks that heavily rely on {non-Markovian} dependencies between layers
(\eg skip connections in residual networks).
In this work, we propose a novel dynamic game perspective
by viewing each \emph{layer} as a \emph{player} in a dynamic game characterized by the DNN itself.
Through this lens,
different classes of optimizers can be seen as matching different types of {Nash equilibria}, %
depending on the implicit information structure {of} each (p)layer.
The resulting method, called Dynamic Game Theoretic Neural Optimizer (DGNOpt),
not only generalizes OCT-inspired optimizers to richer network class;
it
also {motivates} a new training principle by solving a multi-player cooperative game.
DGNOpt shows convergence improvements over existing methods on
image classification datasets with residual and inception networks.
Our work marries strengths from both OCT and game theory, paving ways to new algorithmic opportunities from robust optimal control and {bandit-based optimization}.

\end{abstract}

\section{Introduction}

Attempts from different disciplines to provide a fundamental understanding of deep learning have advanced rapidly in recent years.
Among those, interpretation of DNNs as discrete-time nonlinear dynamical systems
has received tremendous focus.
By viewing each layer as a distinct time step,
it
{motivates} principled analysis from numerical equations \cite{weinan2017proposal,lu2017beyond}
{to physics} \cite{greydanus2019hamiltonian}.
For instance,
casting residual networks \cite{he2016deep} as a discretization of ordinary differential equations
{enables} fundamental reasoning on the loss landscape \cite{lu2020mean}
and inspires new architectures with
numerical stability or continuous limit \cite{chang2018reversible,chen2018neural}.

\begin{figure}[t]
\vskip 0.1in
\begin{center}
\includegraphics[height=3.77cm]{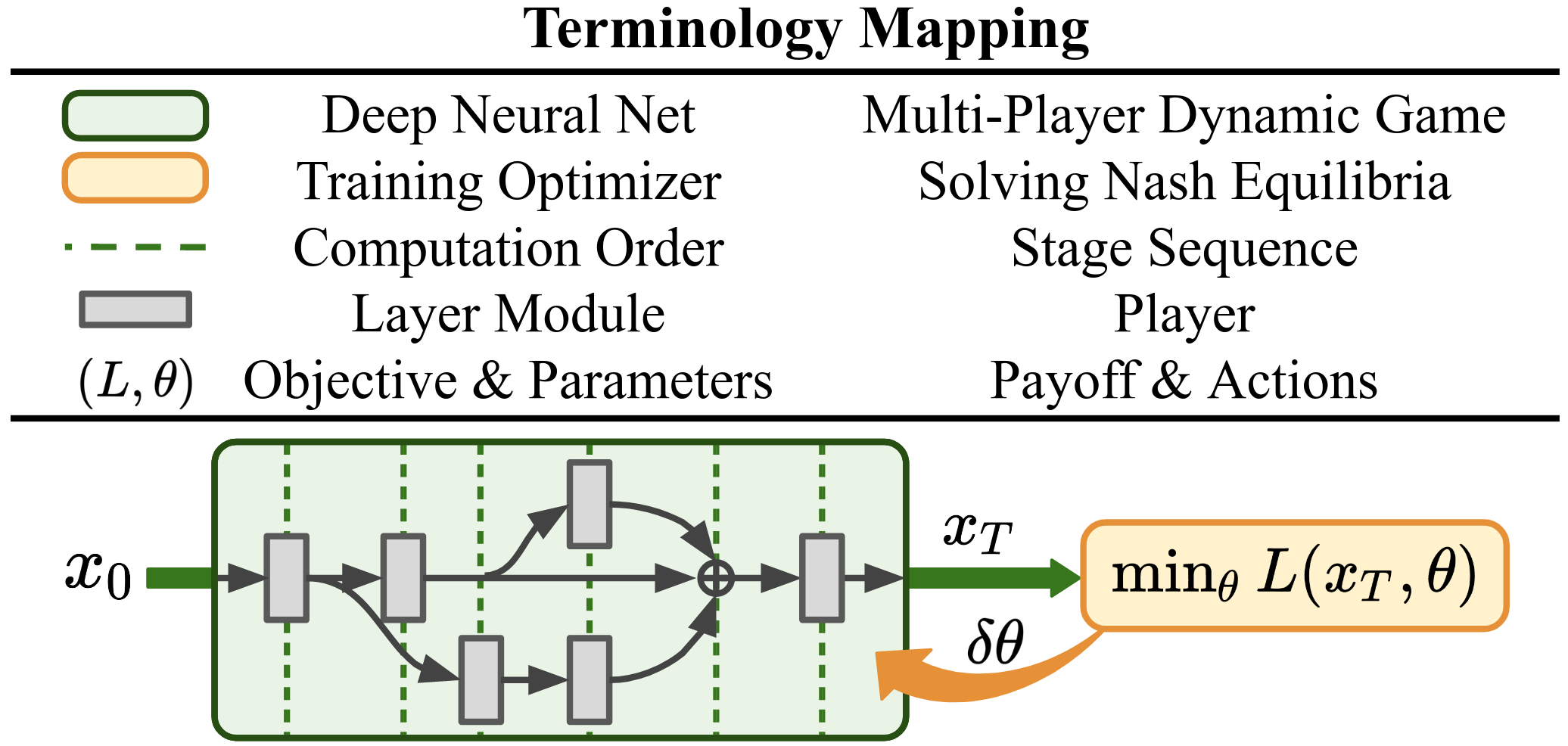}
\vskip -0.1in
\caption{
Dynamic game perspective of generic DNN training process, where we treat layer modules as players
in a dynamic game and solve for the related Nash equilibria (Best viewed in color).
}
\label{fig:1}
\end{center}
\vskip -0.2in
\end{figure}

This dynamical system viewpoint also {motivates} control-theoretic analysis,
which further recasts the network weight as control.
With that, the training process can be viewed as an optimal control problem,
as both methodologies aim to optimize some variables (weights \emph{v.s.} controls) subjected to the {chain structure} (network \emph{v.s.} dynamical system).
This connection has {lead to} theoretical characterization of the learning process \cite{han2018mean,hu2019mean,liu2019deep}
and practical methods for hyper-parameter adaptation \cite{li2017stochastic} or computational acceleration \cite{gunther2020layer,zhang2019you}.

Development of algorithmic progress, however, remains relatively limited.
This is because
OCT-inspired training methods, by construction,
are restricted to network class that
resembles Markovian state-space models \cite{liu2021differential,li2018optimal,li2017maximum}.
This raises questions of their flexibility and scalability to training modern architectures
composed of complex dependencies between layers.
It is unclear whether
this interpretation of dynamical system and optimal control remains suitable, or how it should be adapted, under those cases.

In this work, we address the aforementioned issues using dynamic game theory,
a discipline of interactive decision making \cite{yeung2006cooperative}
built upon optimal control and game theory.
Specifically, we propose to
treat each layer as a player in a dynamic game connected through the network propagation.
The additional dimension gained from multi-player allows us to
generalize OCT-inspired methods to accept a much richer network class.
Further, introducing game-theoretic analysis, \eg \emph{information structure},
provides a novel
algorithmic connection between different classes of training methods from a Nash equilibria standpoint (Fig.~\ref{fig:1}).

Unlike prior game-related works,
which typically cast the whole network as a player competing over training iteration \cite{goodfellow2014generative,balduzzi2018mechanics},
the (p)layers in our dynamic game interact along the network propagation.
This naturally leads to a coalition game since all players share the same objective.
The resulting cooperative training scheme
urges the network to yield group optimality, or Pareto efficiency \cite{pardalos2008pareto}.
As we will show through experiments,
this improves convergence of training modern architectures,
as richer information flows between layers to compute the updates.
We name our method Dynamic Game Theoretic Neural Optimizer (\textbf{DGNOpt}).

Notably,
casting the network as a realization of the game has appeared in analyzing
the convergence of Back-propagation \cite{balduzzi2016deep} or contribution of neurons \cite{stier2018analysing,ghorbani2020neuron}.
Our work instead focuses on developing game-theoretic training methods
and how they can be connected to, or generalize, existing optimizers.
In summary, we present the following contributions.
\begin{itemize}[leftmargin=15pt]
\item
We draw a novel algorithmic characterization from the Nash equilibria perspective by framing the training process as solving a multi-player dynamic game.

\item
We propose \textbf{DGNOpt}, a game-theoretic optimizer that generalizes OCT-inspired methods to richer network class
and encourages {cooperative updates} among layers with an enlarged information structure.

\item
Our method achieves competitive results on image classification with residual and inception nets,
enabling rich applications from robust control and bandit analysis.

\end{itemize}

\section{Preliminaries} \label{sec:prel}

\textbf{Notation:}
Given a real-valued function $\mathcal{F}_s$ indexed by $s~{\in}~\mathcal{S}$,
we shorthand its derivatives evaluated on $(\vx_s,\theta_s)$ as
$\nabla_{\vx_s} \mathcal{F}_s               {\equiv} \mathcal{F}^s_{\vx}$,
$\nabla_{\vx_s}^2 \mathcal{F}_s             {\equiv} \mathcal{F}^s_{\vx\vx}$, and
$\nabla_{\vx_s}\nabla_{\theta_s} \mathcal{F}_s {\equiv} \mathcal{F}^s_{\vx\theta}$, etc.
Throughout this work, we will preserve
$n\in\{1,{\cdots},N\}$ as the player index and
$t\in\{0,1,{\cdots},T{-}1\}$ as the propagation order along the network, or equivalently the stage sequence of the game (see Fig.~\ref{fig:1}).
We will abbreviate them as $n{\in}[N]$ and $t{\in}[T]$ for brevity.
Composition of functions is denoted by $f(g(\cdot)) \equiv (f \circ g)(\cdot)$.
We use $\Inv$, $\odot$ and $\otimes$ to denote pseudo inversion, Hadamard and Kronecker product.
A complete notation table can be found in Appendix~\ref{sec:notation}.

\subsection{Training Feedforward Nets with Optimal Control} \label{sec:ocp-dnn}
Let the layer propagation rule in {feedforward} networks
(\eg fully-connected and convolution networks)
with depth $T$ be
\begin{align}
	\vz_{t+1} =& f_t(\vz_{t}, \theta_{t}),
	\quad t \in [T],
	\label{eq:layer-prop}
\end{align}
where $\vz_t$ and $\theta_t$
represent the vectorized hidden state and parameter at each layer $t$.
For instance, $\theta_t \coloneqq \vectorize([\mW_t, \vb_t])$ for a fully-connected layer,
$f_t(\vz_t,\theta_t) \coloneqq \sigma(\mW_t \vz_t { +} \vb_t)$, with nonlinear activation $\sigma(\cdot)$. %
Equation (\ref{eq:layer-prop}) can be interpreted as a discrete-time Markovian model propagating the state $\vz_t$ with the tunable variable $\theta_t$.
With that, the training process,
\ie finding optimal parameters $\{\theta_t:t{\in}[T] \}$ for all layers, can be described by Optimal Control Programming (OCP),
\begin{align}
	\min_{\theta_t: t\in[T]} L \coloneqq \left[ \phi(\vz_{T}) + \sum_{t=0}^{T-1}\ell_t(\theta_{t}) \right]
	\quad
	\text{\emph{s.t.} (\ref{eq:layer-prop})}  .
	\label{eq:ocp-obj}
\end{align}
The objective $L$ consists of a loss $\phi$ incurred by the network prediction~$\vz_T$ (\eg cross-entropy in classification)
and the layer-wise regularization $\ell_t$ (\eg weight decay).
{Despite (\ref{eq:ocp-obj}) considers only one data point $\vz_0$,
it can be easily modified to accept batch training \citep{han2018mean}.}
Hence, minimizing $L$ sufficiently describes the training process.

Equation (\ref{eq:ocp-obj}) provides an OCP characterization of  training feedforward networks.
First, the {optimality principles} to OCP, according to standard optimal control theory,
typically involve solving some time-dependent objectives recursively from the terminal stage $T$.
Previous works have shown that these backward processes relate closely to the computation of Back-propagation \cite{li2017maximum,liu2021differential}.
Further,
the parameter update of each layer, $\theta_t {\leftarrow} \theta_t {-} \delta \theta_t$, can
be seen as solving these layer-wise OCP objectives with certain approximations.
To ease the notational burden, we leave a thorough discussion in Appendix~\ref{sec:app2}.
We stress that this intriguing connection
is, however, limited to the particular network class described by (\ref{eq:layer-prop}).

\subsection{Multi-Player Dynamic Game (MPDG)} \label{sec:prelim-dyn-game}
Following the terminology in \citet{yeung2006cooperative},
in a discrete-time $N$-player dynamic game, Player $n$ commits to the action $\theta_{t,n}$ at each stage $t$ and seeks to minimize
\begin{align*}
	L_n(\bar{\theta}_n&; \bar{\theta}_{\neg n}) \coloneqq \left[ \phi_n(\vx_{T}) + \sum_{t=0}^{T-1}\ell_{t,n}(\theta_{t,1},{\cdots},\theta_{t,N}) \right]
	\\
	\text{\emph{s.t.} }
	\vx_{t+1} &{=} F_t(\vx_{t}, \theta_{t,1},{\cdots},\theta_{t,N}), \text{ }\text{ }
	{\theta}_{t,n} \equiv {\theta}_{t,n}(\eta_{t,n}),
	\numberthis \label{eq:dya-game-obj}
\end{align*}
where
$\bar{\theta}_n \coloneqq \{ \theta_{t,n}: t \in [T] \}$
denotes the action sequence for Player $n$ throughout the game.
The set $\neg n \coloneqq \{ i {\in} [N]: i {\neq} n \} $ includes all players except Player $n$.
The key components that characterize an MPDG (\ref{eq:dya-game-obj}) are detailed as follows.
\begin{itemize} %
\item
\textbf{Shared dynamics $F_t$.} The stage-wise propagation rule for $\vx_t$, affected by actions across all players $\theta_{t,n}, \forall n$.
\item
\textbf{Payoff/Cost $L_n$.} The objective for each player that accumulates the costs $(\phi_n,\ell_{t,n})$ incurred at each stage.
\item
\textbf{Information structure $\eta_{t,n}$.} A set of information available to Player $n$ at $t$
for making the decision ${\theta}_{t,n}$.
\end{itemize}
The Nash equilibria $\{(\bar{\theta}^*_1,\cdots,\bar{\theta}^*_N)\}$ to (\ref{eq:dya-game-obj}) is a set of stationary points where no player has the incentive to deviate from the decision.
Mathematically, this can be described by
\begin{align*}
	L_n(\bar{\theta}^*_n; \bar{\theta}^*_{\neg n}) \le L_n(\bar{\theta}_n; \bar{\theta}^*_{\neg n}),
	\quad \forall n \in [N], \text{ } \forall \bar{\theta}_n \in \bar{\Theta}_n,
\end{align*}
where $\bar{\Theta}_n$ denotes the set of admissible actions for Player~$n$.
When the players agree to cooperate upon an agreement on a set of strategies and a mechanism to distribute the payoff/cost,
a cooperative game (CG) of (\ref{eq:dya-game-obj}) will be formed.
CG requires additional optimality principles to be satisfied.
This includes
\emph{(i)} group rationality (GR), which requires all players to optimize their joint objective,
\begin{align}
\begin{split}
	L^* \coloneqq
	\min_{\bar{\theta}_1,\cdots,\bar{\theta}_N} \textstyle \sum_{n=1}^N L_n(\bar{\theta}_n; \bar{\theta}_{\neg n}),
	\label{eq:group-ration}
\end{split}
\end{align}
and \emph{(ii)} individual rationality (IR), which requires the cost distributed to each player from $L^*$ be at most the cost he/she will suffer if plays against others non-cooperatively.
Intuitively, IR justifies the participation of each player in CG.

\section{Dynamic Game Theoretic Perspective} \label{sec:3}

\subsection{Formulating DNNs as Dynamic Games} \label{sec:dnn-dg}
In this section, we draw a novel perspective between
the three components in MPDG (\ref{eq:dya-game-obj}) and the training process of generic (\ie non-Markovian) DNNs.
Given a network composed of the layer modules $\{f_i(\cdot,\theta_i)\}$,
where $\theta_i$ denotes the trainable parameters of layer $f_i$ similar to (\ref{eq:layer-prop}),
we treat each layer as a player in MPDG.
The network can be converted into the form of $F_t$ by indexing $i\coloneqq(t,n)$,
where $t$ represents the sequential order from network input to prediction,
and   $n$ denotes the index of layers aligned at $t$.
Fig.~\ref{fig:dy-net} demonstrates such an example for a residual block.
When the network propagation collapses from multiple paths to a single one,
we can consider either duplicated players sharing the same path or dummy players with null action space. %
Hence, \wlg we will treat $N$ as fixed over $t$. %
Notice that the assignment $i\coloneqq(t,n)$ may not be unique.
{We will discuss its algorithmic implication later in $\S$\ref{sec:5-2}.}

Once the shared dynamics is constructed, the payoff
$L_n$ can be readily linked to the training objective.
Since $\ell_{t,n}$ corresponds to the weight decay for layer $f_{t,n}$, it follows that
$\ell_{t,n} \coloneqq \ell_{t,n}(\theta_{t,n})$.
Also, we will have $\phi_{n} \coloneqq \phi$
whenever all (p)layers share the same task,\footnote{
	One of the examples for multi-task objective is the \emph{auxiliary loss} used in deep reinforcement learning \cite{jaderberg2016reinforcement}.
} \eg in classification.
In short, the network architecture and training objective respectively characterize the structure of a dynamic game and its payoff.

\begin{figure}[t]
\vskip 0.05in
\begin{center}
\includegraphics[width=0.9\linewidth]{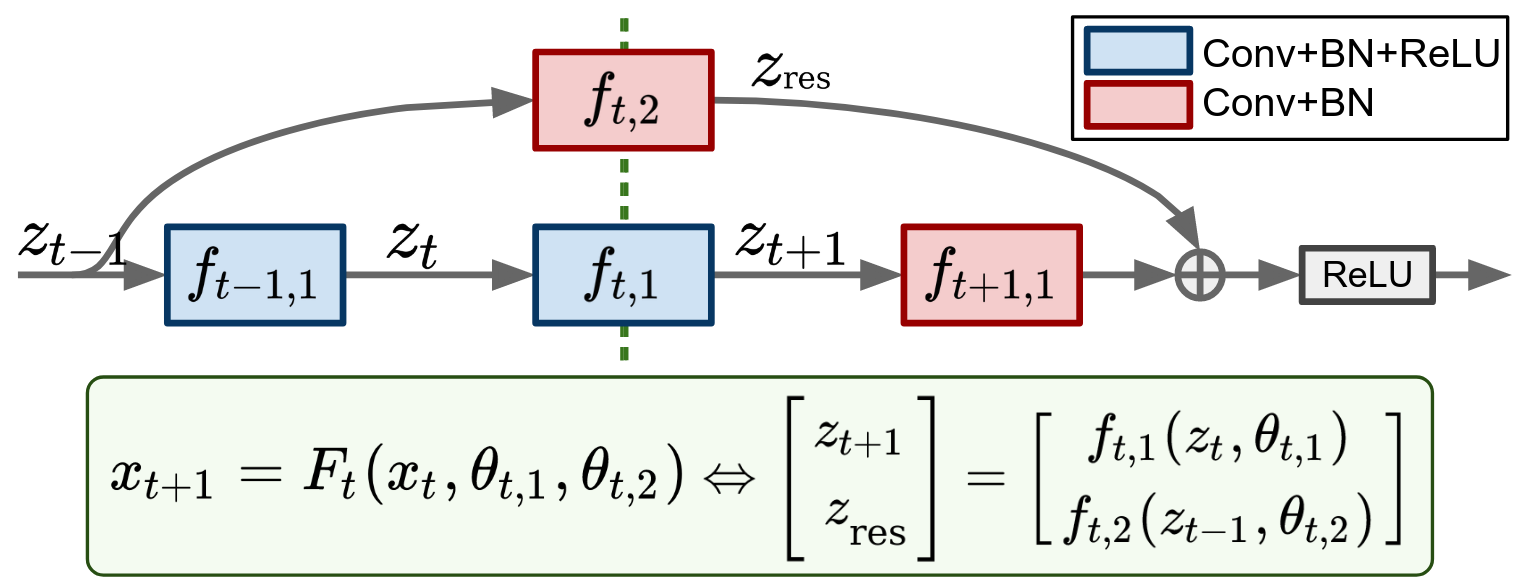}
\vskip -0.1in
\caption{Example of representing a residual block as $F_t$.
Note that $\vx_t$ augments all hidden states across parallel paths.
}
\label{fig:dy-net}
\end{center}
\vskip -0.1in
\end{figure}

\subsection{Information Structure and Nash Optimality} \label{sec:3.2}

We now turn into the role of information structure $\eta_{t,n}$.
Standard game-theoretic analysis suggests that $\eta_{t,n}$ determines the type of Nash equilibria inherited in the MPDG \cite{petrosjan2005cooperative}.
Below we introduce several variants that are of our interests,
starting from the one with the least structure.
\begin{definition}[Open-loop Nash equilibrium (OLNE)]
	Let $\etaOL~{\coloneqq}~\{\vx_0\}$ be the open-loop information structure.
	Then a set of action,
	$\{{\theta}^*_{t,n}: \forall t, n\}$,
	provides an OLNE to (\ref{eq:dya-game-obj})
	if
	\begin{align*}
		\theta^*_{t,n} = \argmin_{\theta_{t,n}} \text{ } &H_{t,n}(\vx_t,\vp_{t+1,n},\theta_{t,n}, \theta^*_{t,\neg n}),
		\numberthis \label{eq:olne} \\
		\text{where }
		\theta^*_{t,n} \equiv \theta^*_{t,n}(&\etaOL) \text{ and }
		H_{t,n} \coloneqq \ell_{t,n} + F_t^\transpose \vp_{t+1,n}
	\end{align*}
	is the Hamiltonian for Player $n$ at stage $t$.
	The co-state $\vp_{t,n}$ is a vector of the same size as $\vx_t$ and can be simulated from the backward adjoint process,
	$({\vp}_{t,n},\vp_{T,n}) {\coloneqq} (H^{t,n}_{\vx}, \phi^n_{\vx})$.
\end{definition}
The Hamiltonian objective $H_{t,n}$ varies for each $(t,n)$ and depends on the proceeding stage via co-state $\vp_{t{+}1,n}$.
When $N{=}1$, (\ref{eq:olne}) degenerates to the celebrated Pontryagin principle \citep{pontryagin1962mathematical},
which provides the necessary condition to OCP (\ref{eq:ocp-obj}).
This motivates the following result.
\begin{proposition}
	\label{prop:nash-open-loop}
	Solving
	$\theta_{t,n}^{*}{=}\argmin H_{t,n}$
	with the iterative update, $\theta_{t,n}\leftarrow\theta_{t,n}{-}\mM^\Inv H^{t,n}_{\theta}$, recovers the descent direction of standard training methods.
	Specifically, setting
	\bgroup
	\arraycolsep=3pt
	\begin{align*}
		\mM {\coloneqq}
		\left\{\begin{array}{l} {\mI} \\ {\diag\left(\sqrt{\Hutn \odot \Hutn}\right)} \\ {\Hutn\Hutn^\transpose}  \end{array}\right.
		\text{yields}
		\left\{\begin{array}{l} \textnormal{SGD} \\ \textnormal{RMSprop} \qquad . \\ \textnormal{Gauss-Newton} \end{array}\right.
	\end{align*}
	\egroup
\end{proposition}
Proposition \ref{prop:nash-open-loop}
provides a similar OCP characterization (\emph{c.f.} $\S$\ref{sec:ocp-dnn})
except for a more generic network class represented by $F_t$.
It also gives
our first game-theoretic interpretation of DNN training:
\emph{standard training methods implicitly match an OLNE defined upon the network propagation}.
The proof (see Appendix \ref{sec:app3})
relies on constructing a set of co-state $\vp_{t,n}$ such that $H^{t,n}_{\theta}{\equiv}\nabla_{\theta_{t,n}} H_{t,n}$ gives the exact gradient w.r.t. the parameter of layer $f_{t,n}$.
The degenerate information structure $\etaOL$ implies that optimizers of this class utilize minimal knowledge available from the game (\ie network) structure.
This is in contrast to the following Nash equilibrium which relies on richer information.

\bgroup
\def\arraystretch{1.0}
\begin{table}[t]
	\begin{center}
	\caption{Dynamic game theoretic perspective of DNN training.
	}
	\label{table:game-dnn}
	\vskip 0.1in
	\begin{small}
	\begin{tabular}{>{\hspace{-3pt}}c<{\hspace{-2pt}}!{\vrule width 0.7pt}>{\hspace{-2pt}}c<{\hspace{-3pt}}c<{\hspace{-3pt}}c<{\hspace{-3pt}}}
	\toprule
	\specialcell{ Nash \\ Equilibria } &
	\specialcell{Information \\ Structure } &
	\specialcell{Optimality \\ Principle } &
	\specialcell{Class of \\ Optimizer} \\ [1pt]
	\midrule %
	OLNE & $\etaOL$  & $\min H_{t,n}$  in (\ref{eq:olne}) & Baselines \\ [3pt]
	FNE  & $\etaCL$  & $\min Q_{t,n}$  in (\ref{eq:i-bellman}) & { DGNOpt (ours)} \\ [3pt]
	GR   & $\etaCCG$ & $\min P_t$  in (\ref{eq:gr-bellman}) & {DGNOpt (ours)} \\
	\bottomrule
	\end{tabular}
	\end{small}
	\end{center}
	\vskip -0.1in
\end{table}
\egroup
\begin{definition}[Feedback Nash equilibrium (FNE)]
	Let $\etaCL {\coloneqq} \{\vx_s: s~{\le}~t \}$ be the closed-loop information structure.
	Then a set of strategy,
	$\{ \pi^*_{t,n}: \forall t,n \}$,
	is called a FNE to (\ref{eq:dya-game-obj})
	if it is the solution to the {Isaacs-Bellman} equation (\ref{eq:i-bellman}).
	\begin{align}
	\begin{split}
	  V_{t,n}(\vx_t) = \min_{\pi_{t,n}}\text{ } &Q_{t,n}(\vx_t,\pi_{t,n},\pi^*_{t,\neg n}), \\
	  V_{T,n} = \phi_n, \quad \text{where} \quad
	  &Q_{t,n} \coloneqq \ell_{t,n} + V_{t+1,n} \circ F_t
	\end{split} \label{eq:i-bellman}
	\end{align}
	is the {Isaacs-Bellman} objective for Player $n$ at stage $t$.
	Also, $\pi_{t,n} \equiv \theta_{t,n}(\vx_t; \etaCL)$ denotes any arbitrary mapping from $\vx_t$ to $\theta_{t,n}$,
	conditioned on the closed-loop structure $\etaCL$.
\end{definition}
For the closed-loop information structure $\etaCL$,
each player has complete access to all preceding states until the current stage $t$.
Consequently,
it is preferable to solve for a state-dependent, \ie \emph{feedback}, {strategy} $\pi^*_{t,n}$ rather than a state-independent action $\theta^*_{t,n}$ as in OLNE.
Similar to (\ref{eq:olne}), the Isaacs-Bellman objective $Q_{t,n}$ is constructed for each $(t,n)$,
except now carrying a \emph{function} $V_{t,n}(\cdot)$ backward from the terminal stage, rather than the co-state.
This \emph{value function} $V_{t,n}$ summarizes the optimal cost-to-go for Player $n$ from each state $\vx_t$,
provided all afterward stages are minimized accordingly.
When $N{=}1$, (\ref{eq:i-bellman}) collapses to standard Dynamic Programming (DP; \citet{bellman1954theory}),
which is an alternative optimality principle parallel to the Pontryagin.
For nontrivial $N{>}1$,
solving the FNE optimality~(\ref{eq:i-bellman})
provides a game-theoretic extension for previous DP-inspired training methods,
\eg \citet{liu2021differential}, to generic (\ie non-Markovian) architectures.

\subsection{Cooperative Game Optimality}

Now, let us consider the CG formulation. %
When a cooperative agreement is reached, each player will be aware of how others react to the game.
This can be mathematically expressed by the following information structures,
\begin{align*}
	\etaOCG \coloneqq \{ \vx_0, \theta^*_{t,\neg n} \} \text{ } \text{ and } \text{ }
	\etaCCG \coloneqq \{ \vx_s, \pi^*_{t,\neg n}: s \le t \},
\end{align*}
which enlarge $\etaOL$ and $\etaCL$ with additional knowledge from other players, $\neg n$.
We can characterize the inherited optimality principles similar to OLNE and FNE.
Take $\etaCCG$ for instance, {the joint optimization in GR (\ref{eq:group-ration}) requires }
\begin{definition}[Cooperative feedback solution]
	A set of strategy,
	$\{ \pi^*_{t,n} : \forall t,n \}$,
	provides an optimal feedback solution to the joint optimization (\ref{eq:group-ration})
	if it solves

	\begin{align*}
	  &W_{t}(\vx_t) = \min_{\pi_{t,n}:n\in[N]} P_{t}(\vx_t,\pi_{t,1},{\cdots},\pi_{t,N}), \numberthis \label{eq:gr-bellman} \\
	  W_{T} &= \textstyle\sum_{n=1}^N \phi_n, \quad \text{where } P_{t} \coloneqq \textstyle\sum_{n=1}^N \ell_{t,n} + W_{t+1} \circ F_t
	\end{align*}
	is the ``group-rational'' Bellman objective at stage $t$.
	$\pi_{t,n} \equiv \theta_{t,n}(\vx_t; \etaCCG)$ denotes arbitrary mapping from $\vx_t$ to $\theta_{t,n}$,
	conditioned on the cooperative closed-loop structure $\etaCCG$.
\end{definition}
Notice that (\ref{eq:gr-bellman}) is the GR extension of (\ref{eq:i-bellman}).
Both optimality principles solve for a set of feedback strategies,
except the former considers a joint objective $P_t$ summing over all players.
Hence, it is sufficient to carry a joint value function $W_t$ backward.
We leave the discussion on $\etaOCG$ in Appendix~\ref{sec:app3}.

To emphasize the importance of information structure,
consider the architecture in Fig.~\ref{fig:mnist-a}, where
each pair of parallel layers shares the same input and output;
hence the network resembles a two-player dynamic game with $F_t \coloneqq f_{t,1}{+}f_{t,2}$.
As shown in Fig.~\ref{fig:mnist-b}, providing different information structures to the same optimizer, EKFAC \cite{george2018fast},
greatly affects the training.
Having richer information tends to achieve better performance.
Additionally, the fact that
\begin{align}
	\etaOL \subset \etaCL \subset \etaCCG
	\quad \text{and} \quad \etaOL \subset \etaOCG \subset \etaCCG
	\label{eq:subset-relation}
\end{align}
also implies an algorithmic connection between different classes of optimizers,
which we will explore in $\S$\ref{sec:alg-connect}.

Table \ref{table:game-dnn} summarizes our game-theoretic analysis.
Each information structure suggests its own Nash equilibria and optimality principle,
which characterizes a distinct class of training methods.
We already established the connection between baselines and $H_{t,n}$ in Proposition~\ref{prop:nash-open-loop}.
In the next section, we will derive methods for solving $(Q_{t,n},P_t)$.

\begin{figure}[t]%
  \centering
  \subfloat{
  	\includegraphics[trim=15 5 0 0,clip,width=0.95\linewidth]{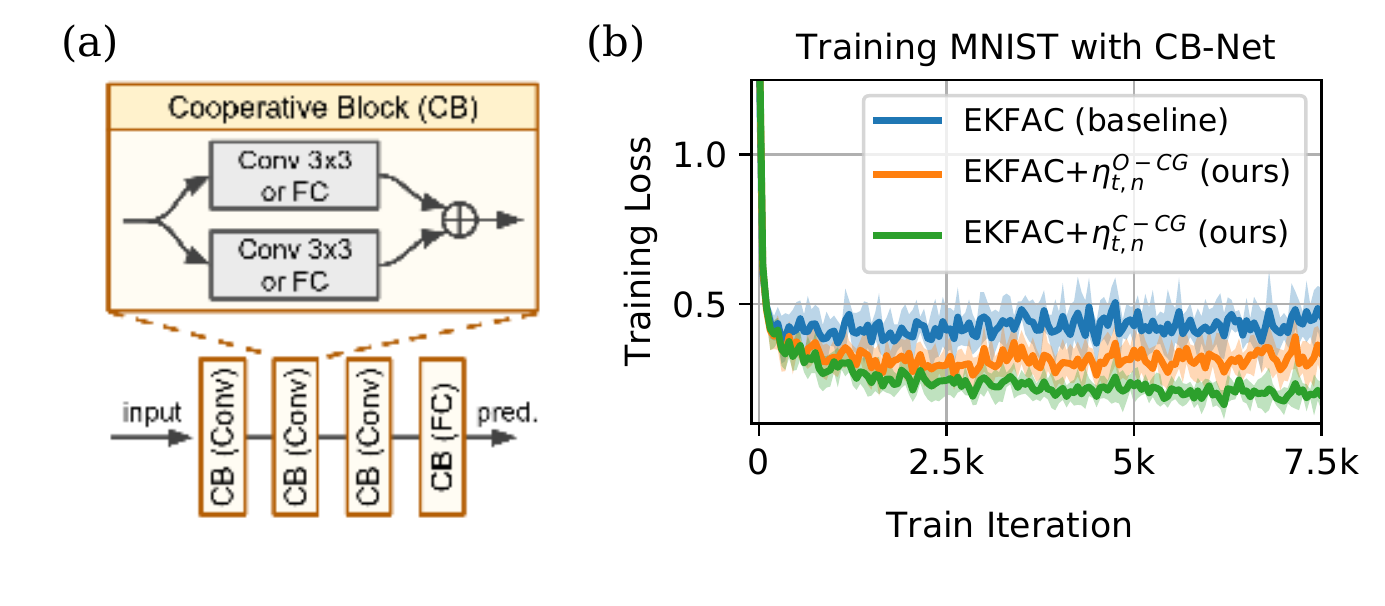}
  	\label{fig:mnist-a}%
  }
  \subfloat{
  	\textcolor{white}{\rule{1pt}{1pt}}
  	\label{fig:mnist-b}%
  }
  \vskip -0.23in
  \caption{
  	(a) The cooperative-block network and (b) its training performance on MNIST when the optimizer
	(EKFAC) is exposed to different information structure $\eta_{t,n}$.
	Note that by Proposition \ref{prop:nash-open-loop}, the EKFAC baseline utilizes only the open-loop structure $\etaOL$.
  }
  \label{fig:mnist}
  \vskip -0.25in
\end{figure}

\section{Training DNN by Solving Dynamic Game} \label{sec:4}

In this section, we derive a new \emph{second-order} method, called
Dynamic Game Theoretic Neural Optimizer (\textbf{DGNOpt}),
that solves (\ref{eq:i-bellman}) and (\ref{eq:gr-bellman})
as an alternative to training DNNs.
While we will focus on
the residual network for its
{popularity and algorithmic simplicity when deriving the analytic update},
we stress that our methodology applies to other architectures.
A full derivation is left in Appendix \ref{sec:app4}.

\subsection{Iterative Update via Linearization}

Computing the game-theoretic objectives $(Q_{t,n},P_t)$
requires knowing
$(F_t,\ell_{t,n},\phi_n)$.
Despite they are well-defined from $\S$\ref{sec:dnn-dg},
carrying $Q_{t,n}$ or $P_t$ as a stage-varying function
is computationally impractical even on a relatively low-dimensional system \cite{tassa2012synthesis},
let alone DNNs.
Since the goal is to derive an {incremental} update given partial (\eg mini-batch) data at each training iteration,
we can consider solving them approximately via \emph{linearization}.

Iterative methods via linearization have been widely used in real-time OCP \cite{pan2015robust,tassa2014control}.
We adopt a similar methodology for its computational efficiency and algorithmic connection to existing training methods (shown later).
First, consider solving the FNE recursion~(\ref{eq:i-bellman}) by ${\pi}^*_{t,n}{\approx}\argmin Q_{t,n}$.
We begin by performing second-order Taylor expansion on $Q_{t,n}$ w.r.t. to the variables that are \emph{observable}
to Player $n$ at stage $t$ according to $\etaCL$.
\bgroup
\arraycolsep=4pt
\begin{align*}
		Q_{t,n} \approx
		{\frac{1}{2}}
		\left[\begin{array}{c} {\mathbf{1}} \\ {\dvx_t} \\ {\dth_{t,n}} \end{array}\right]^{\transpose}
		\left[\begin{array}{lll}
			{Q_{t,n}}    & {\Qxt^\transpose} & {\Qut^\transpose} \\
			{\Qxt} & {\Qxxt} & {\Quxt^\transpose} \\
			{\Qut} & {\Quxt} & {\Quut}
			\end{array}\right]
		\left[\begin{array}{c} {\mathbf{1}} \\ {\dvx_t} \\ {\dth_{t,n}} \end{array}\right]
\end{align*}
\egroup
Note that $\delta \theta_{t,\neg n}$ does not appear in the above quadratic expansion since
it is unobservable according to $\etaCL$.
The derivatives of $Q_{t,n}$ w.r.t. different arguments follow standard chain rule (recall $Q_{t,n} \coloneqq \ell_{t,n} + V_{t+1,n} \circ F_t$),
with the dynamics linearized at some fixed point $(\vx_t,\theta_{t,n})$, \eg
\begin{align*}
	\Qut = \lutn + \FutT \Vxtn, \text{ } \Quxt = \FutT \Vxxtn \Fxt.
\end{align*}

The analytic solution to this quadratic expression is given by $\pi^*_{t,n} {=} \theta_{t,n} {-} \delta \pi^*_{t,n}$, with
the incremental update $\delta \pi^*_{t,n}$ being
\begin{align}
\begin{split}
	\delta \pi_{t,n}^* = \ktn &+ \Ktn \dvx_{t}. \\
	\ktn \coloneqq (\Quutn)^\Inv \Qutn \text{ }\text{ }\text{ } \text{and}& \text{ }\text{ }\text{ } \Ktn \coloneqq (\Quutn)^\Inv \Quxtn
	\label{eq:fne-update}
\end{split}
\end{align}
are called the open and feedback gains.
The superscript $^\Inv$ denotes the pseudo inversion.
Note that $\delta \pi^*_{t,n}$ is only \emph{locally} optimal around the region where the quadratic expansion remains valid.
Since $\vx_{t}$ augments preceding hidden states (\eg $\vx_{t} {\coloneqq} [\vz_{t},\vz_{t-1}]^\transpose$ in Fig.~\ref{fig:dy-net}), (\ref{eq:fne-update})
implies that preceding hidden states contribute to the update via linear superposition.

Substituting the incremental update $\delta \pi^*_{t,n}$ back to the FNE recursion~(\ref{eq:i-bellman}) yields the local expression of the value function $V_{t,n}$, which will be used to compute the preceding update $\delta \pi^*_{t-1,n}$.
Since the computation
depends on $V_{t,n}$ only through its local derivatives $V^{t,n}_\vx$ and $V^{t,n}_{\vx\vx}$,
it is sufficient to propagate these quantities rather than the function itself.
The propagation formula is summarized in (\ref{eq:fne-vxx}).
This procedure ({line 4-7 in Alg.~\ref{alg:dgnopt}}) repeats recursively backward from the terminal to initial stage, similar to Back-propagation.
\begin{align}
\begin{split}
	V_\vx^{t,n} &= Q_\vx^{t,n} {-} Q_{\vx\theta}^{t,n} \ktn,\quad \text{ } V_\vx^{T,n} = \phi_\vx^n, \\
	V_{\vx\vx}^{t,n} &= Q_{\vx\vx}^{t,n} {-} Q_{\vx\theta}^{t,n} \Ktn,\quad V_{\vx\vx}^{T,n} = \phi_{\vx\vx}^n.
	\label{eq:fne-vxx}
\end{split}
\end{align}

\begin{algorithm}[t]
   \caption{Dynamic Game Theoretic Neural Optimizer} %
   \label{alg:dgnopt}
\begin{algorithmic}[1]
   \STATE {\bfseries Input:} dataset $\mathcal{D}$, network $F\equiv \{F_t: t\in[T]\}$
   \REPEAT
   \STATE Compute $\vx_t$ by propagating $\vx_0 \sim \mathcal{D}$ through $F$
   \FOR{$t=T{-}1$ {\bfseries to} $0$ }{\hfill \markgreen{$\rhd$ Solve FNE or GR}}
       \STATE Solve the update $\delta \pi^*_{t,n}$ with (\ref{eq:fne-update}) or (\ref{eq:gr-update})
	   \STATE Solve {${(V_\vx^{t,n}{,}V_{\vx\vx}^{t,n})}$} or {${(W_\vx^t{,}W_{\vx\vx}^t)}$} with (\ref{eq:fne-vxx}) or (\ref{eq:gr-vxx})
   \ENDFOR
   \STATE Set ${\vx}^\prime_0=\vx_0$
   \FOR{$t=0$ {\bfseries to} $T{-}1$}{\hfill \markgreen{$\rhd$ Update parameter}}
	   \STATE Apply $\theta_{t,n} {\leftarrow} \theta_{t,n} {-} \delta \pi^*_{t,n}(\dvx_{t})$ with $\dvx_{t} {=} {\vx}^\prime_{t} {-} \vx_{t}$
	   \STATE Compute ${\vx}^\prime_{t+1} = F_t({\vx}^\prime_t, \theta_{t,1}, \cdots, \theta_{t,N})$ %
   \ENDFOR
   \UNTIL{ converges }
\end{algorithmic}
\end{algorithm}

\def\Gu{{\tilde{\mathbf{G}}}}
\def\II{{\mathbf{I}}}
\def\LL{{\mathbf{L}}}

Derivation for CG follows similar steps
except we consider solving the GR recursion (\ref{eq:gr-bellman}) by ${\pi}^*_{t,n}{\approx}\argmin P_{t}$.
Since all players' actions are now observable from $\etaCCG$,
we need to expand $P_t$ w.r.t. all arguments.
For notational simplicity, let us denote $\vu \equiv \theta_{t,1}, \vv \equiv \theta_{t,2}$
in {Fig.~\ref{fig:dy-net}}.
In the case when each player minimizes $P_t$ independently without knowing the other,
we know the non-cooperative update for Player~2 admits the form\footnote{
    Similar to (\ref{eq:fne-update}), we have $\It \coloneqq \PvvInvt\Pvt$ and $\Lt \coloneqq \PvvInvt\Pvxt$.
} of $\dvv_t = \It + \Lt \dvx_t$.
Now, the locally-optimal cooperative update for Player~1 can be written as
\begin{align*}
	&\delta \pi_{t,1}^* = \kut + \Kut \dvx_t, \text{ where} \numberthis \label{eq:gr-update} \\
	\kut \coloneqq  \PuuCInvt (\Put  &- \Puvt \It), \text{ }
	 \Kut \coloneqq  \PuuCInvt (\Puxt - \Puvt \Lt).
\end{align*}
Similar equations can be derived for Player~2. %
We will refer
$\PuuCt {\coloneqq} \Puut {-} \Puvt\PvvInvt\Pvut$
as the \emph{cooperative precondition}.
The update (\ref{eq:gr-update}), despite seemly complex,
exhibits intriguing properties.
For one, notice that
computing the cooperative open gain $\kut$ for Player~1 involves the non-cooperative open gain $\It$ from Player~2.
In other words, each player adjusts the strategy after knowing the companion's action.
Similar interpretation can be drawn for the feedbacks $\Kut$ and $\Lt$.
Propagation of $(W_\vx^{t},W_{\vx\vx}^{t})$
follows similarly as (\ref{eq:fne-vxx}) once all players' updates are computed.
We leave the complete formula in (\ref{eq:gr-vxx}) (see Appendix \ref{sec:C1}) since it is rather tedious.

Let us discuss the role of $\dvx_t$ and how to compute them.
Conceptually, $\dvx_t$ can be any deviation away from the fixed point $\vx_{t}$ where we expand the objectives,
$Q_{t,n}$ or $P_t$.
In MPDG application,
it is typically set to the state difference when the parameter updates are applied until stage {$t$},
\begin{align}
    \dvx_{t} \coloneqq (F_{t{-}1} {\circ} \cdots {\circ} F_0)
                         (\vx_0, \{\theta + \delta {\pi}^*\}_{{<} t, {\forall} n} ) -
                         \vx_{t} ,
    \label{eq:dx}
\end{align}
where $\{\delta {\pi}^*\}_{< t, \forall n} \coloneqq \{ \delta {\pi}_{s,n}^*: s {<} t, \forall n \}$
collects all players' updates until stage $t$.
In this view, the feedback
compensates all changes, including those that may cause instability, cascading from the preceding layers;
hence it tends to robustify the training process \cite{de1988differential,liu2021differential}. %

Alg.~\ref{alg:dgnopt} presents the pseudo-code of DGNOpt, which consists of
\emph{(i)} the same forward propagation through the network (line 3), \emph{(ii)} a distinct game-theoretic backward process that solves either FNE or GR optimality (line 4-7), and \emph{(iii)} an additional forward pass that applies the feedback updates $\delta \pi^*_{t,n}$ (line 8-12; also Fig.~\ref{fig:coop-update}).
We stress that Alg.~\ref{alg:dgnopt} accepts
any generic DNN so long as it can be represented by $F_t$.

\begin{figure}[t]
\vskip 0.05in
\begin{center}
\includegraphics[width=0.935\linewidth]{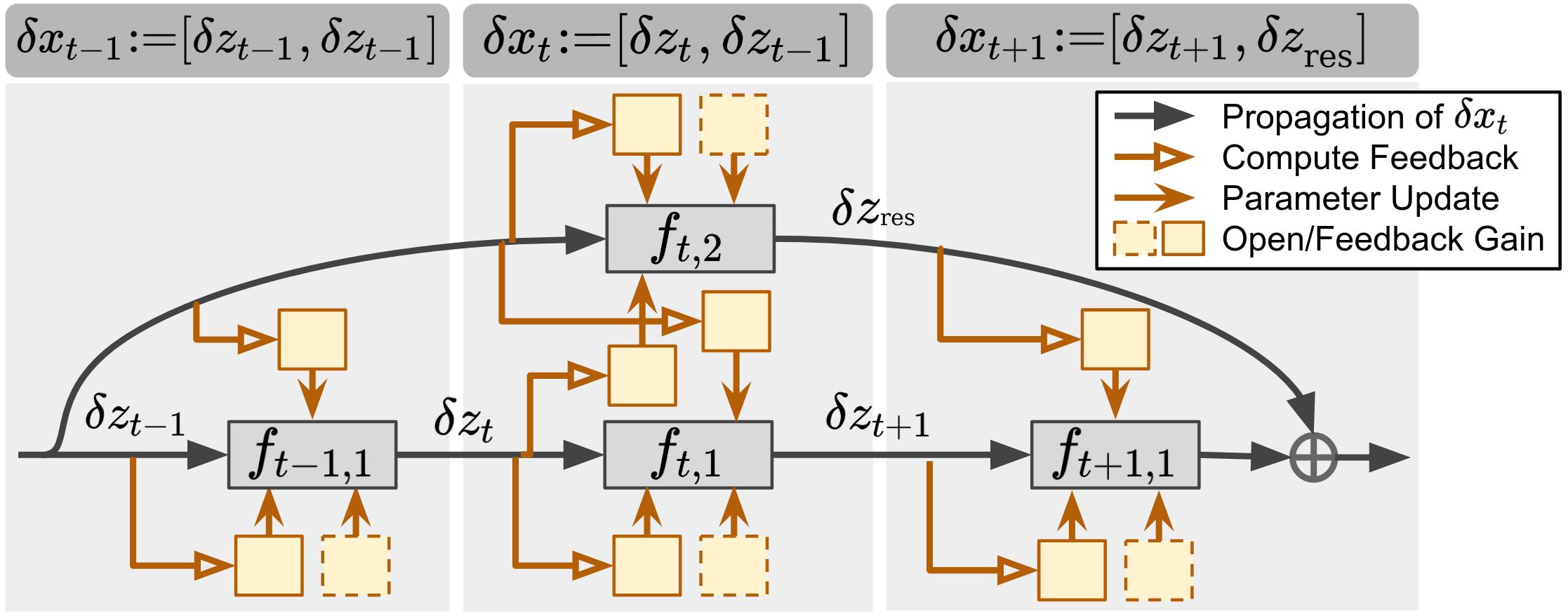}
\vskip -0.125in
\caption{
How $\delta \pi^*_{t,n}$ is applied (\emph{c.f.} line 8-12 in Alg.~\ref{alg:dgnopt}) to a residual block.
We compute $\dvx_t$ with a forward propagation (\ref{eq:dx}) and simultaneously update the parameter.
The solid yellow box denotes the feedback dependent on the preceding hidden states $\vz_{s\le t}$.
}
\label{fig:coop-update}
\end{center}
\vskip -0.1in
\end{figure}

\subsection{Curvature Approximation} \label{sec:4.2}

Naively inverting the parameter curvature, \ie
$(\Quutn)^\Inv$ and $\PuuCInvt$, can be computationally inefficient and sometimes unstable
for practical training.
To mitigate the issue,
we adopt curvature amortizations \cite{kingma2014adam,hinton2012neural} used in DNN training.
These methods naturally fit into our framework by recalling Proposition~\ref{prop:nash-open-loop} that
different baselines differ in how they estimate the curvature $H^{t,n}_{\theta\theta}$ for the preconditioned update $H^{t.n\text{ }\Inv}_{\theta\theta} \Hutn$.
With this in mind, we can estimate the FNE parameter curvature $\Quutn$ with
\begin{align}
	\Quutn \approx \Qutn\Qutn^\transpose \text{ }\text{ or }\text{ } \diag( \sqrt{\Qutn \odot \Qutn} ),
    \label{eq:q-approx}
\end{align}
which resembles the Gauss-Newton (GN) matrix or its adaptive diagonal matrix (as appeared in RMSProp and Adam).

As for $\PuuCInvt$, which contains an \emph{inner} inversion $\PvvInvt$ inside $\PuuCt$,
we propose a new approximation inspired by
the Kronecker factorization (KFAC; \citet{martens2015optimizing}).
KFAC factorizes the GN matrix with two smaller-size matrices.
We leave the complete discussion on KFAC, as well as the proof of the following result, in Appendix~\ref{sec:C2}.
\begin{proposition}[KFAC for $\PuuCt$] \label{prop:coop}
	Suppose
	$\Puut$ and $\Pvvt$ are factorized
	with some vectors $\vz_1,\vz_2,\vg_1,\vg_2$ by
	\begin{align*}
	  \Puut &\approx \Axx \otimes \Bxx \eqqcolon \Auu \otimes \Buu, \\
	  \Pvvt &\approx \Ayy \otimes \Byy \eqqcolon \Avv \otimes \Bvv,
	\end{align*}
    where the expectation is taken over the mini-batch data.
	Let
	$\Auv \coloneqq \Axy$ and
	$\Buv \coloneqq \Bxy$,
	then the cooperative precondition matrix in (\ref{eq:gr-update}) can be factorized by
	\begin{align*}
	  \PuuCt
	  &\approx \AuuC \otimes \BuuC \numberthis \label{eq:coop-kfac} \\
	  &= (\Auu-\Auv\AvvInv\AuvT) \otimes (\Buu-\Buv\BvvInv\BuvT).
	\end{align*}

\end{proposition}
In practice, we set
$(\vz_1,\vz_2,\vg_1,\vg_2){ \coloneqq} (\vz_{t},\vz_{t-1},W^{t+1}_{\vz_t},W^{t+1}_{\vz_{t-1}})$
for the residual block in Fig.~\ref{fig:dy-net} or \ref{fig:coop-update}.
With Proposition \ref{prop:coop}, we can compute the update, take $\kut$ for instance, by
\begin{align}
	\kut
	=\vectorize(\BuuCInv(\Put-\Buv \vectorize^\Inv(\It) \AuvT  )\AuuCInvT),
    \label{eq:coop-open}
\end{align}
where $\vectorize^\Inv$ is the inverse operation of vectorization ($\vectorize$).

Another computation source comes from the curvature w.r.t. the MPDG state,
\ie $V_{\vx\vx}^{t,n}$ and $W_{\vx\vx}^{t}$.
Here, we approximate them with low-rank matrices using either Gauss-Newton or their top eigenspace.
These are rather reasonable approximations since
it has been constantly observed that these Hessians are highly degenerate for DNNs
\cite{wu2020dissecting,papyan2019measurements,sagun2017empirical}.
With all these, we are able to train modern DNNs by solving their corresponding dynamic games,
(\ref{eq:i-bellman}) or (\ref{eq:gr-bellman}),
with a runtime comparable to other first and second-order methods (see {Fig.~\ref{fig:runtime}}).

\addtocounter{footnote}{-1}

\bgroup
\begin{figure*}[t]
    \vskip -0.05in
    \centering
    \begin{minipage}{0.6\textwidth}
        \centering
        \captionsetup{type=table}
        \captionsetup{justification=centering}
        \caption{Accuracy (\%) of residual-based networks (averaged over $6$ random seeds)}
        \label{table:clf}
        \vskip 0.04in
        \begin{small}
        \begin{tabular}{l!{\vrule width 1pt}cccc!{\vrule width 0.5pt}c!{\vrule width 0.5pt}c} %
        \toprule
        \multirow{2}{*}{Dataset} &
        \multicolumn{5}{c|}{Baselines (\ie $\min H_{t,n}$ {in} OLNE)}
        & \multirow{2}{*}{\specialcell{\textbf{DGNOpt} \\ \textbf{(ours)}}} \\ [1pt]
        & SGD & RMSProp & Adam & \multicolumn{1}{c}{EKFAC} & {EMSA} & \\
        \midrule
        {MNIST}     &  98.65  &  98.61 &  98.49  &  \textbf{98.77}  & 98.25 & {98.76}  \\
        {SVHN}      &  88.58  &  88.96 &  89.20  &  88.75  & 87.40  &  \textbf{89.22}
        \\
        {CIFAR10}   &  82.94 &  83.75  &  85.66  &  85.65  & {75.60}   & \textbf{85.85}  \\
        {CIFAR100}  &  71.78  &  71.65 &  71.96  &  71.95  & {62.63} &  \textbf{72.24} \\ %
        \bottomrule
        \end{tabular}
        \end{small}
        \vskip -0.08in
        \centering
        \captionsetup{type=table}
        \captionsetup{justification=centering}
        \caption{Accuracy (\%) of inception-based networks (averaged over $4$ random seeds)}
        \label{table:clf2}
        \vskip 0.04in
        \begin{small}
        \begin{tabular}{l!{\vrule width 1pt}cccc!{\vrule width 0.5pt}c!{\vrule width 0.5pt}c} %
        \toprule
        \multirow{2}{*}{Dataset} &
        \multicolumn{5}{c|}{Baselines (\ie $\min H_{t,n}$ {in} OLNE)}
        & \multirow{2}{*}{\specialcell{\textbf{DGNOpt} \\ \textbf{(ours)}}} \\ [1pt]
        & SGD & RMSProp & Adam & \multicolumn{1}{c}{EKFAC} & {EMSA} & \\
        \midrule
        {MNIST}     &  97.96 & 97.75 & 97.72 &  97.90  & 97.39 & \textbf{98.03}  \\
        {SVHN}      &  87.61 &  86.14 &  86.84 &  88.89  & 82.68 &  \textbf{88.94} \\
        {CIFAR10}   &  76.66 &  74.38 &  75.38 &  77.54  & {70.17}   & \textbf{77.72}  \\
        \bottomrule
        \end{tabular}
        \end{small}
    \end{minipage}
    $\text{ }$
    \begin{minipage}{0.38\textwidth}
        \vskip 0.05in
        \centering
        \includegraphics[width=0.93\columnwidth]{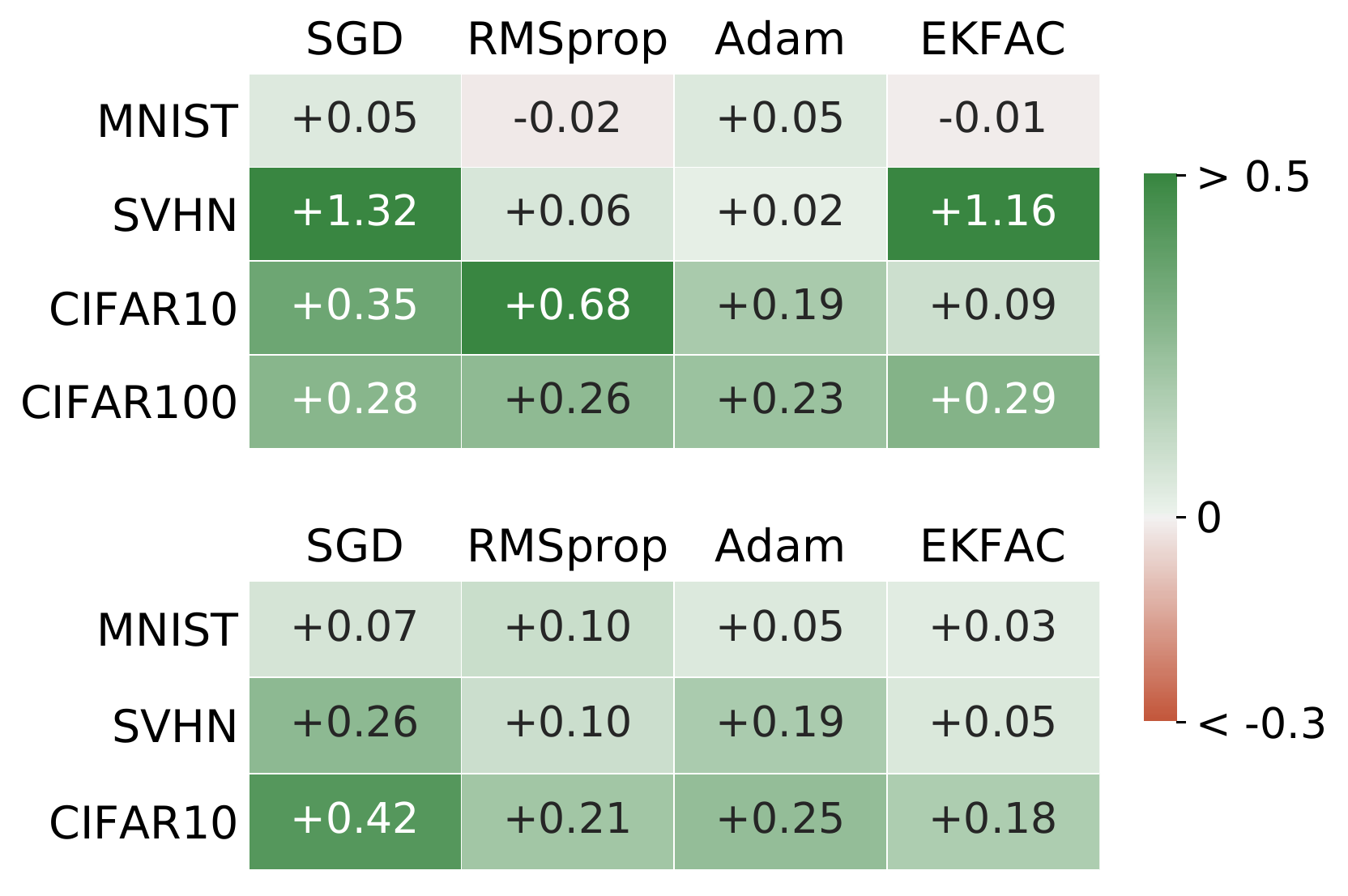}
        \vskip -0.15in
        \caption{
        Accuracy (\%) improvement (\markgreen{+}) or degradation (\markred{-}) when richer information structure,
        \ie $\etaOL {\rightarrow} \{\etaCL,\etaCCG\}$,
        is used for each best-tuned baseline\protect\footnotemark $\text{}$ in Table~\ref{table:clf} (upper) and Table~\ref{table:clf2} (bottom). Color bar is scaled for best view.
        }
        \label{fig:abl-als}
    \end{minipage}
\vskip -0.05in
\end{figure*}
\egroup

\subsection{Algorithmic Connection} \label{sec:alg-connect}

Finally, let us discuss an intriguing algorithmic equivalence.
Recall the {subset} relation among the information structures in (\ref{eq:subset-relation}).
Manipulating these structures allows one to traverse between
different game optimality principles.
For instance, masking $\pi_{t,\neg n}^*$ in $\etaCCG$
makes it degenerate to $\etaCL$,
which implies the FNE and GR optimality become equivalent.
Through this lens,
one may wonder if a similar algorithmic relation can be drawn for these iterative updates.
This is indeed the case as shown below (proof left in Appendix~\ref{sec:C3}).
\begin{theorem}[Algorithmic equivalence] \label{thm:alg-connection}
	$\text{ }$
	\begin{itemize}
	\item  (\ref{eq:gr-update}) with $\Puvt\coloneqq\mathbf{0}$ gives (\ref{eq:fne-update})
	\item  (\ref{eq:fne-update})  with $(\Quxtn,\Quutn)\coloneqq(\mathbf{0},\mI)$
	 gives SGD
	\item  (\ref{eq:gr-update}) with $(\Puvt,\Puxt,\Puut)\coloneqq(\mathbf{0},\mathbf{0},\mI)$ gives SGD
	\end{itemize}
	Setting $\Quutn$ and $\Puut$ to other precondition matrices, similar to (\ref{eq:q-approx}), recovers other baselines.
\end{theorem}
The intuition behind Theorem \ref{thm:alg-connection} is that when higher-order (${>}2$) expansions are discarded,
setting $\Puvt\coloneqq\mathbf{0}$ completely blocks the communication between two players;
therefore we effectively remove $\pi_{t,\neg n}^*$ from $\etaCCG$.
Similarly, forcing $\Quxtn \coloneqq \mathbf{0}$ prevents Player $n$ from observing
how changing $\vx_t$ may affect the payoff, hence one can at best achieve the same OLNE optimality as baselines.
Theorem \ref{thm:alg-connection} implies that (\ref{eq:fne-update}) and (\ref{eq:gr-update}) generalize standard updates
to richer information structure; thereby creating more complex updates.

\section{Experiment} \label{sec:experiment}

\subsection{Evaluation on Classification Datasets} \label{sec:5.1}
\textbf{Datasets and networks.}
We verify the performance of DGNOpt on image classification datasets as they are suitable testbeds for modern networks that contain
non-Markovian dependencies.
Specifically,
we first consider residual-based networks
given their popularity and our thorough discussions in $\S$\ref{sec:4}.
For larger datasets such as CIFAR10/100, we train ResNet18 with multi-stepsize learning rate decay.
For MNIST and SVHN, we use residual networks composed of 3 residual blocks (see Fig.~\ref{fig:dy-net}).
Meanwhile,
we also consider inception-based networks, which composed of an inception block (see Fig.~\ref{fig:inception}) that resembles a 4-player dynamic game.
All networks use ReLU activation and are trained with 128 batch size.
Other setups are detailed in Appendix~\ref{sec:app5}.

\begin{figure}[t]
  \vskip -0.1in
  \begin{center}
    \includegraphics[width=0.58\linewidth]{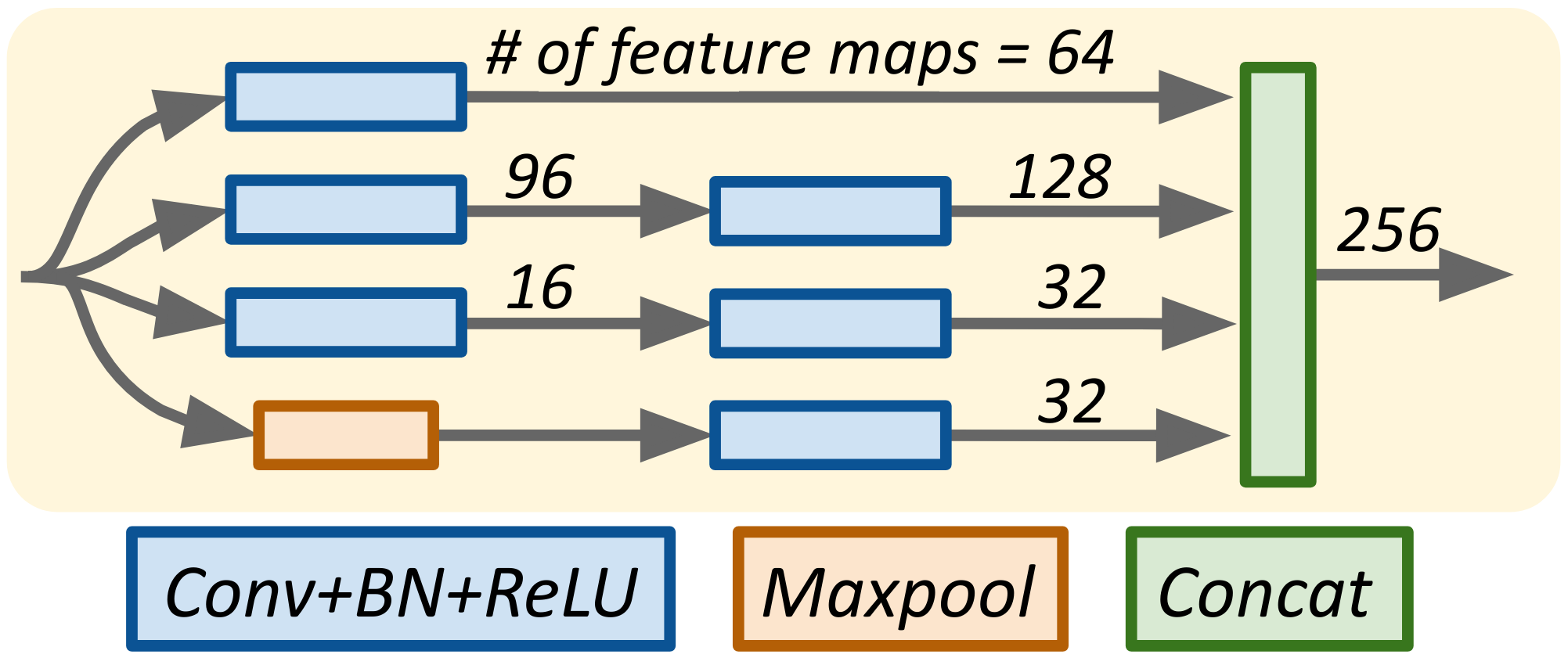}
  \end{center}
  \vskip -0.2in
  \caption{Architecture of the inception block.}
  \label{fig:inception}
  \vskip -0.15in
\end{figure}

\textbf{Baselines.}
Motivated by our discussion in $\S$\ref{sec:3},
we compare DGNOpt, which essentially solves FNE and GR, %
with methods involving OLNE either implicitly or explicitly.
This includes standard training methods such as
SGD,
RMSprop, Adam, and EKFAC \citep{george2018fast}, which is an extension to the second-order method KFAC
with eigenvalue-correction.
To also compare against OCT-inspired methods,
we {include} EMSA \citep{li2017maximum}, which {explicitly} minimizes a {modified} Hamiltonian.
Other OCT-based training methods mostly consider degenerate, \eg discrete-weighted \cite{li2018optimal} or Markovian \cite{liu2021differential}, networks.
In this view, DGNOpt generalizes those methods to both larger network class and richer information structure.

\textbf{Performance and ablation study.}
Table \ref{table:clf} and \ref{table:clf2} summarize the performance for the residual and inception networks.
On most datasets,
DGNOpt achieves competitive results against standard methods and outperforms EMSA by a large margin.
Despite both originates from the OCT methodology,
{in practice EMSA often exhibits numerical instability for larger networks}. %
On the contrary,
DGNOpt leverages iteration-based linearization and amortized curvature, which
greatly stabilizes the training.

On the other hand,
DGNOpt distinguishes itself from standard baselines by
considering a larger information structure.
To validate the benefit of having this additional knowledge during training,
we conduct an ablation study using the algorithmic connection built in Theorem \ref{thm:alg-connection}.
Specifically, we measure the performance difference when the best-tuned baselines, \ie the ones we report in Table \ref{table:clf} and \ref{table:clf2}, are further allowed to utilize higher-level information.
Algorithmically, this can be done by running DGNOpt with the parameter curvature replaced
by the precondition matrix of each baseline.
For instance, replacing all $Q_{\theta\theta}^{t,n}$ with identity matrices~$\mI$ while keeping other computation unchanged
is {equivalent} to lifting SGD to accept the closed-loop structure $\etaCL$.
From Theorem~\ref{thm:alg-connection},
these two training procedures now differ only in the presence of $Q_{\theta\vx}^{t,n}$,
which allows SGD to adjust its update based on the change of $\vx_t\in\etaCL$.
As shown in Fig.~\ref{fig:abl-als},
enlarging the information structure tends to enhance the performance,
or at least being innocuous.
We highlight these improvements as the benefit gained from introducing dynamic game theory
to the original OCT interpretation.

\footnotetext{
    The ablation analysis in Fig.~\ref{fig:abl-als} applies Theorem~\ref{thm:alg-connection}
    to methods that solve the exact Hamiltonian;
    hence excludes EMSA since it instead considers a modified Hamiltonian (see Appendix~\ref{sec:app5}).
}

\textbf{Overhead vs performance trade-offs.}
As shown in Fig.~\ref{fig:runtime},
DGNOpt enjoys a comparable runtime and memory complexity to standard methods on training ResNet18.
Specifically, its per-iteration runtime is around $\pm$40\% compared to the second-order baseline,
depending on the information structures (DGNOpt-FNE \emph{v.s.} DGNOpt-GR).
In practice, these gaps tend to vanish for smaller networks. %
The overhead introduced by DGNOpt enables the computation of \emph{feedback updates} using a richer information structure.
From the OCT standpoint,
the feedback is known to play a key role in compensating the unstable disturbance along the propagation.
Particularly,
when problems inherit chained constraints (\eg DNNs),
these feedback-enhanced methods often converge faster with superior numerical stability against standard methods \citep{murray1984differential}.

To validate the role of feedback in training modern DNNs, notice that
one shall expect the effect of feedback becomes significant
when a larger step size is taken.
This is because (see (\ref{eq:dx})) larger $\delta \pi^*$ increases $\delta \vx_t$, which amplifies the
feedback $\mathbf{K}\delta\vx_t$.
Fig.~\ref{fig:large-lr} confirms our hypothesis,
where we train the inception-based networks on MNIST and SVHN using relatively large learning rates.
It is clear that utilizing feedback updates greatly stabilizes the training.
While the SGD baseline struggles to make stable progress, DGNOpt converges almost flawlessly (with negligible overhead).
As for well-tuned hyperparameter which often has a smaller step size, our ablation analysis in Fig.~\ref{fig:abl-als}
suggests that having feedbacks throughout the stochastic training generally leads to better local minima.

\begin{figure}[t]
    \vskip 0.05in
    \centering
    \begin{minipage}{0.48\textwidth}
        \begin{center}
        \includegraphics[height=2.9cm]{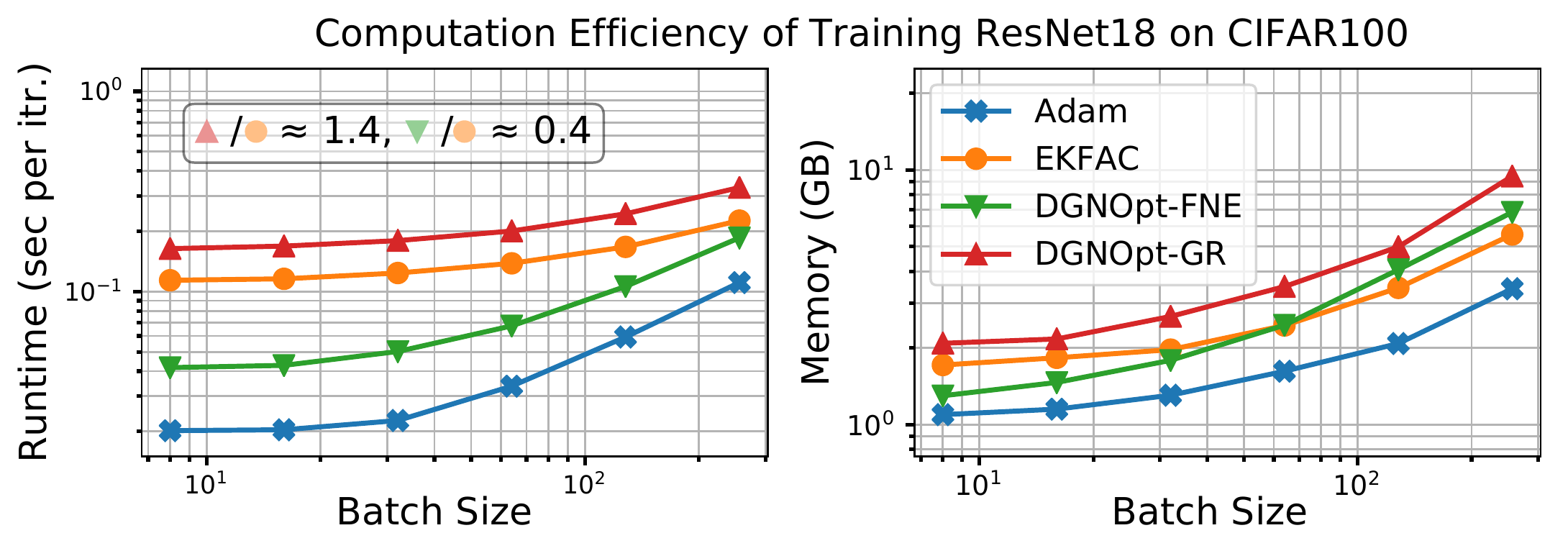}
        \vskip -0.15in
        \caption{
            Our second-order method {DGNOpt} exhibits similar runtime ($\pm$40\%) and memory ($\pm$30\%) complexity compared to the second-order baseline EKFAC.
            }
        \label{fig:runtime}
        \end{center}
    \end{minipage}
    \vskip 0.1in
    \begin{minipage}{0.48\textwidth}
        \begin{center}
        \includegraphics[width=\columnwidth]{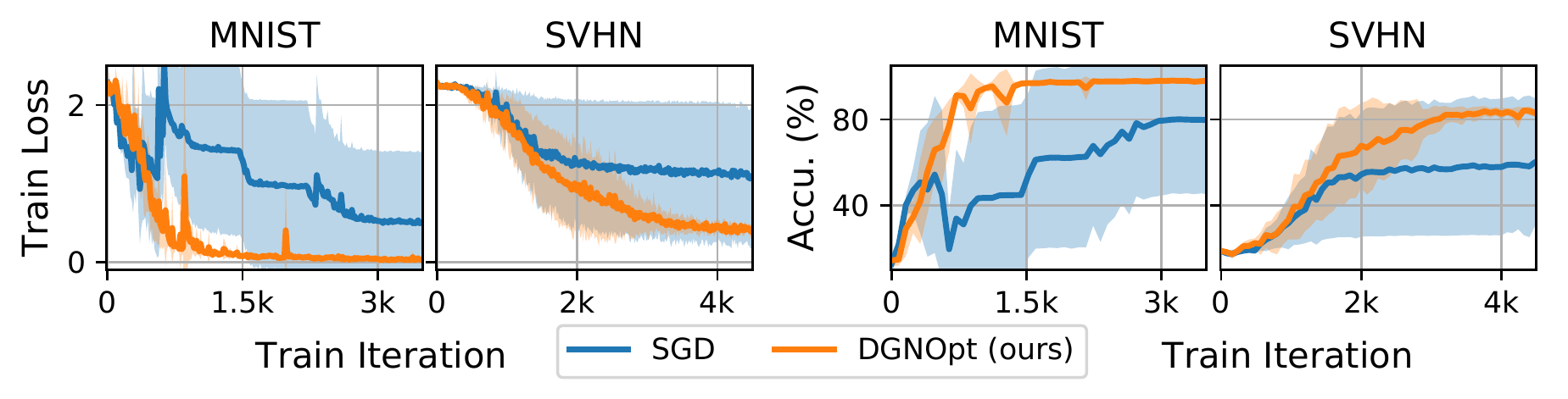}
        \vskip -0.15in
        \caption{
            Training inception-based networks using larger step sizes, where MNIST (\resp SVHN) uses
            $\mathrm{lr}{=}$1.0 (\resp $\mathrm{lr}{=}$2.0).
            }
        \label{fig:large-lr}
        \end{center}
    \end{minipage}
\vskip -0.05in
\end{figure}

\subsection{Game-Theoretic Applications} \label{sec:5-2}

\textbf{Cooperative training with fictitious agents.}
Despite all the rigorous connection we have explored so far,
it is perhaps unsatisfactory to see
our multi-agent {analysis} %
degenerates when facing
feedforward networks, since the number of player $N$ becomes trivially 1. %
We can remedy this scenario by considering the following transformation.
\begin{align}
    F_t(\vz_t,\theta_{t,1}, {\cdots},\theta_{t,N}) \coloneqq f_t(\vz_t,\theta_t), \text{ } \textstyle\sum_{n=1}^N \theta_{t,n} {=} \theta_t
    \label{eq:robust-traj}
\end{align}
In other words, we can divide the layer's weight (or player's action) into multiple parts,
so that the MPDG framework remains applicable.
Interestingly, the transformation of this kind resembles game-theoretic robust optimal control \cite{pan2015robust,sun2018min},
where the controller (or player in our context)
models external disturbances with \emph{fictitious agents},
in order
to enhance the robustness or convergence of the optimization process.

Fig.~\ref{fig:multi-player} and Table~\ref{table:multi-player} provide
the training results when SGD presumes different numbers of players interacting in a {feedforward} network
consisting of 4 convolution and 2 fully-connected layers (see Appendix~\ref{sec:app5}).
Notice that $N{=}1$ corresponds to the original method.
For $N{>}1$, we apply the transformation (\ref{eq:robust-traj}) then solve for the cooperative update as in DGNOpt.
We stress that these fictitious agents only appear during the training phase for computing the cooperative updates.
At inference, actions from all players collapse back to $\theta_t$ by the summation in (\ref{eq:robust-traj}).

While it is clear that
encouraging agents to cooperate during training
can achieve better minima at a faster rate,
having more agents, surprisingly, does not always imply better performance.
In practice, the improvement can slow down or even degrade once $N$ passes some critical values.
This implies that $N$ shall be treated as a {hyper-parameter} of {these game-extended methods}.
Empirically, we find that $N{=}2$ provides a good trade-off between the final performance and convergence speed.
We observe a consistent result for this setup on other optimizers (see {Appendix \ref{sec:app5}} for EKFAC).

\begin{figure}[t]
    \vskip -0.05in
    \centering
    \begin{minipage}{0.485\textwidth}
        \centering
        \captionsetup{type=table}
        \captionsetup{justification=centering}
        \caption{
            Convergence speed w.r.t. $N$.
            Numerical values report the training steps required to achieve certain accuracy on each dataset.
        }
        \label{table:multi-player}
        \vskip 0.1in
        \begin{small}
        \begin{tabular}{l!{\vrule width 0.7pt}cccc}
        \toprule
        Achieved & \multicolumn{4}{c}{Number of Player ($N$) \textnormal{in} SGD} \\
        Performance & $1$ & $2$ & $4$ & $6$ \\
        \midrule %
        80\% \textnormal{in} SVHN   & 5.14k & 2.31k & 1.25k & \textbf{0.8k} \\
        60\% \textnormal{in} CIFAR10 & 3.62k & \textbf{2.97k} & \textbf{2.98k} & 5.83k \\
        \bottomrule
        \end{tabular}
        \end{small}
    \end{minipage}
    \vskip 0.1in
    \begin{minipage}{0.485\textwidth}
        \begin{center}
        \subfloat{
        \includegraphics[height=2.85cm]{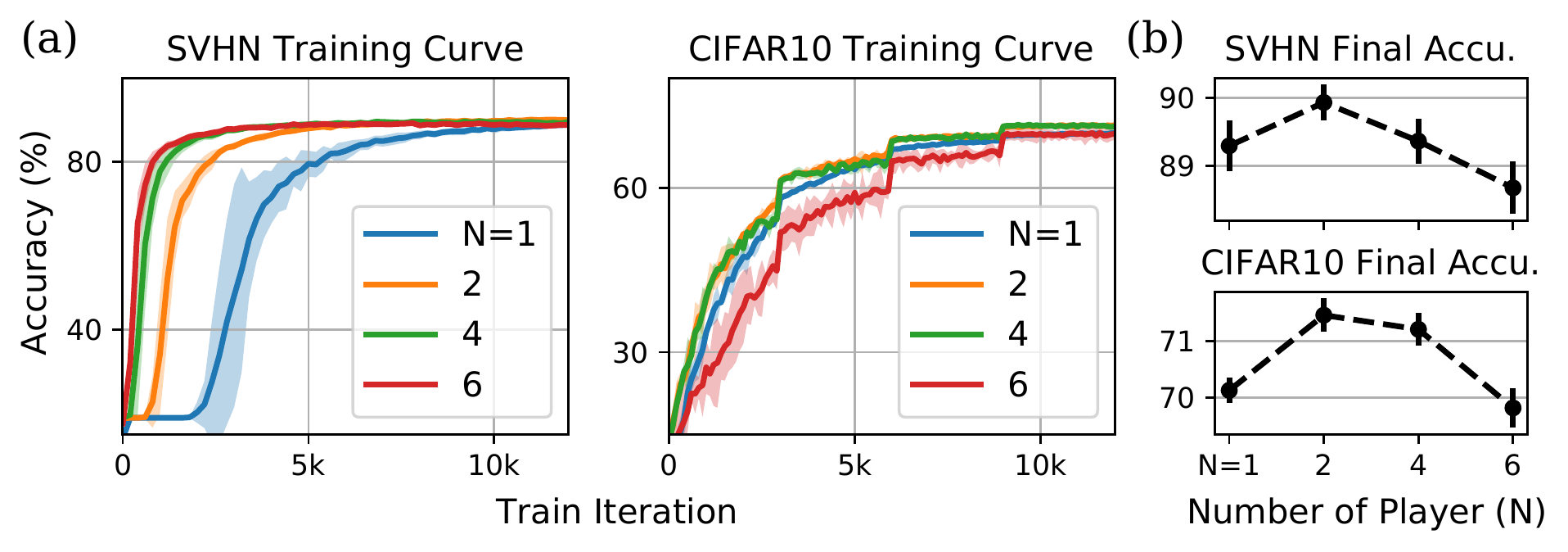}
        \label{fig:multi-player-a}%
        }
        \subfloat{
        \textcolor{white}{\rule{1pt}{1pt}}
        \label{fig:multi-player-b}%
        }
        \vskip -0.1in
        \caption{
            (a) Training curve and (b) final accuracy as we vary the number of player ($N$)
            as a hyperparameter of game-extended SGD.
        }
        \label{fig:multi-player}
        \end{center}
    \end{minipage}
\vskip -0.11in
\end{figure}

\textbf{Adaptive alignment using multi-armed bandit.}
Finally, let us discuss an application of the bandit algorithm in our framework.
In $\S$\ref{sec:dnn-dg}, we briefly mentioned that mapping from modern networks to the shared dynamics $F_t$ most likely will not be unique.
For instance (see Fig.~\ref{fig:bandit-a}), placing the shortcut module of a residual block at different locations leads to different $F_t$;
hence results in different DGNOpt updates.
This is a distinct feature arising exclusively from our MPDG framework,
since these alignments are unrecognizable to standard baselines.
It naturally raises the following questions:
what is the optimal strategy to align the (p)layers of the network in our dynamic game,
and how do different aligning strategies affect training?

To answer these questions,
we compare the performance between three strategies,
including
\emph{(i)} using a fixed alignment throughout training,
\emph{(ii)} random alignment at each iteration, and
\emph{(iii)} adaptive alignment using a multi-armed bandit.
For the last case,
we interpret pulling an arm as selecting one of the alignments
and associate the round-wise reward with the validation accuracy at each iteration.
Note that this is a \emph{non-stationary} bandit problem since the reward distribution of each arm/alignment
evolves as we train the network.
We provide the pseudo-code of this procedure in {Appendix \ref{sec:app5}}.

Fig.~\ref{fig:bandit-b} and Table~\ref{table:bandit} report the results of DGNOpt using different aligning strategies.
We also include the baseline when the information structure shrinks from $\etaCCG$ to
$\etaOL$,
similar to the ablation study in $\S$\ref{sec:5.1}.
In this case, all these DGNOpt variants degenerate to EKFAC.
For the non-stationary bandit, we find EXP3++ \cite{seldin2014one} to be sufficient in this application.
While DGNOpt with fixed alignment already achieves faster convergence compared with the baseline,
dynamic alignment using either random or adaptive strategy leads to further improvement (see Fig.~\ref{fig:bandit-b}).
Notably, having the adaptation throughout training also enhances the final accuracy.
For CIFAR10 with ResNet18, the value is boost by 1\% from baseline and 0.5\% compared with the other two strategies.
This sheds light on new algorithmic opportunities inspired by \emph{architecture-aware} optimization.

\begin{figure}[t]
    \vskip -0.05in
    \centering
    \begin{minipage}{0.485\textwidth}
        \centering
        \captionsetup{type=table}
        \captionsetup{justification=centering}
        \caption{Training result (accuracy \%) of Fig.~\ref{fig:bandit} on two datasets.}
        \label{table:bandit}
        \vskip 0.1in
        \begin{small}
        \begin{tabular}{l!{\vrule width 0.7pt}cccc}
        \toprule
        \multirow{2}{*}{Dataset} & \multirow{2}{*}{EKFAC} & \multicolumn{3}{c}{DGNOpt + Aligning Strategy} \\ %
         &  & fixed & random & adaptive \\
        \midrule
        SVHN    & 87.49 & 88.20 & 88.12 & \textbf{88.33} \\
        CIFAR10 & 84.67 & 85.20 & 85.27 & \textbf{85.65} \\
        \bottomrule
        \end{tabular}
        \end{small}
    \end{minipage}
    \vskip 0.1in
    \begin{minipage}{0.485\textwidth}
        \begin{center}
            \subfloat{
            \centering
            \includegraphics[height=2.8cm]{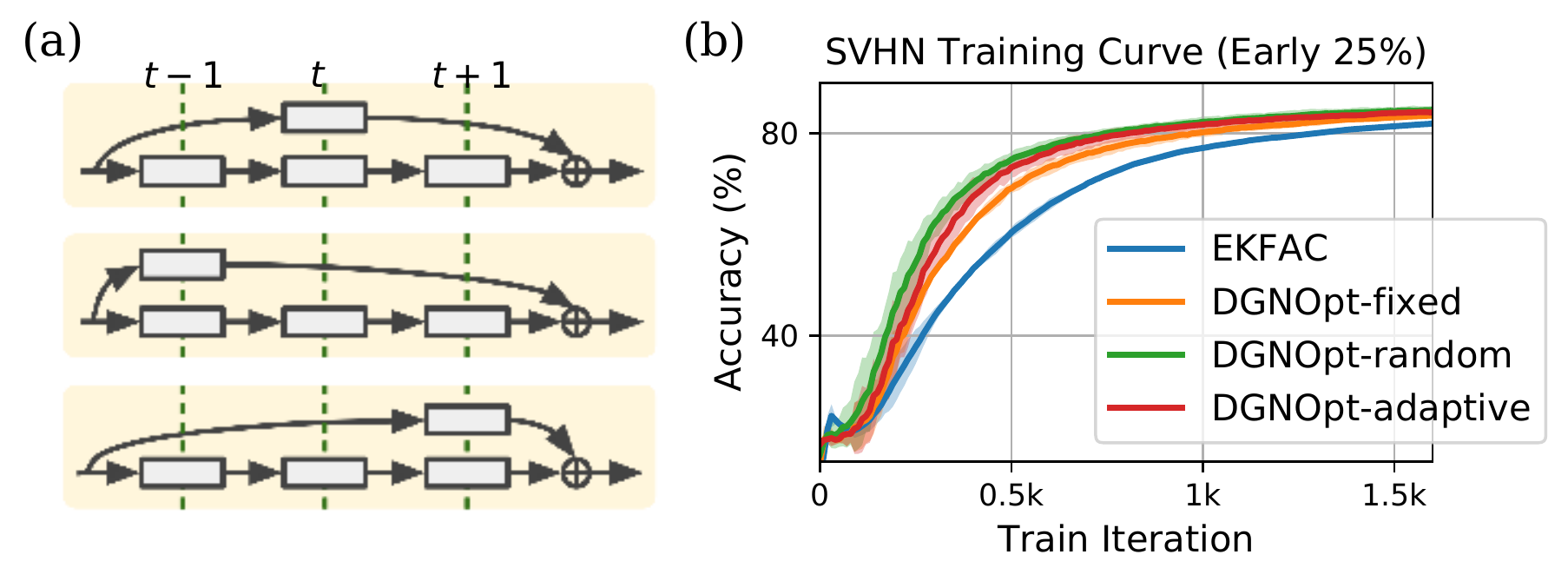}
            \label{fig:bandit-a}%
            }
            \subfloat{
            \textcolor{white}{\rule{1pt}{1pt}}
            \label{fig:bandit-b}%
            }
            \vskip -0.1in
            \caption{
                (a) Different alignments of the same residual block
                leads to distinct DGNOpt updates, yet they are unrecognizable to baselines.
                (b) Early phase of training with different aligning strategies.
            }
            \label{fig:bandit}
            \end{center}
    \end{minipage}
\vskip -0.15in
\end{figure}

\section{Discussion} \label{sec:discussion}

\newcommand{\marklink}[1]{{\color{red} #1}}

\textbf{Comparison to Markovian-based OCT-inspired methods.}
As we briefly mentioned in $\S$\ref{sec:3.2}, our DGNOpt (with FNE) can be seen as a game-theoretic extension of \citet{liu2021differential},
which is also an OCT-inspired method despite concerning only Markovian networks.
It is natural to wonder whether these two methods are interchangeable
since one can always force a non-Markovian system to be Markovian by lifting it into higher dimensions or aggregating the state.

Here,
we stress that our DGNOpt differs from \citet{liu2021differential} in many significant ways.
For one, forming a Markovian chain by grouping the non-Markovian layers into higher dimensions
leads to a \emph{degenerate} information structure.
The (p)layers inside each Markovian group, $\{{f}_{t,n}: {t}^- {<} t {<} {t}^+ \}$, only have access to
$\vx_{s\le{t^-}}$ rather than full latest information $\vx_{s\le t}$ as in DGNOpt,
since their dependencies are discarded.
From the Nash standpoint,
this leads to degenerate backward optimality and update rules.
Indeed, in the limit when we simply group the whole network as single-step dynamics, we will recover $\etaOL {\triangleq} \{\vx_{0}\}$ in baselines.
In contrast, DGNOpt
fully leverages the structural relation of the network,
hence enables rich game-based applications, \eg bandit or robust control, that are otherwise infeasible with \citet{liu2021differential}.

\begingroup
\setlength{\columnsep}{10pt}%
\textbf{Degeneracy when partitioning parameters as players.}
In $\S$\ref{sec:5-2}, we demonstrate a specific transformation, \ie (\ref{eq:robust-traj}),
that makes cooperative training possible for single-player feedforward networks
while respecting our \emph{layer-as-player} game formulation.
This transformation may seem artificial at first glance compared to a naive alternative that directly partitions the parameters of each layer as distinct players.
Unfortunately, the latter strategy yields \emph{degenerate} cooperative optimality.
To see it, notice that treating the $\pi_{t,n}$ appeared in the joint optimization~(\ref{eq:gr-bellman}) as the $n^\text{th}$-partitioned parameters of layer~$t$
is \emph{equivalent} to solving the FNE optimality~(\ref{eq:i-bellman}) with $N{=}1$ (so that the $\pi_{t,N{=}1}$ in (\ref{eq:i-bellman}) becomes the intact parameters of layer $t$).
Hence, it collapses to the prior single-player non-cooperative method \citep{liu2021differential},
\begin{wrapfigure}[8]{r}{0.26\textwidth}
  \vspace{-8pt}
  \centering\includegraphics[width=\linewidth]{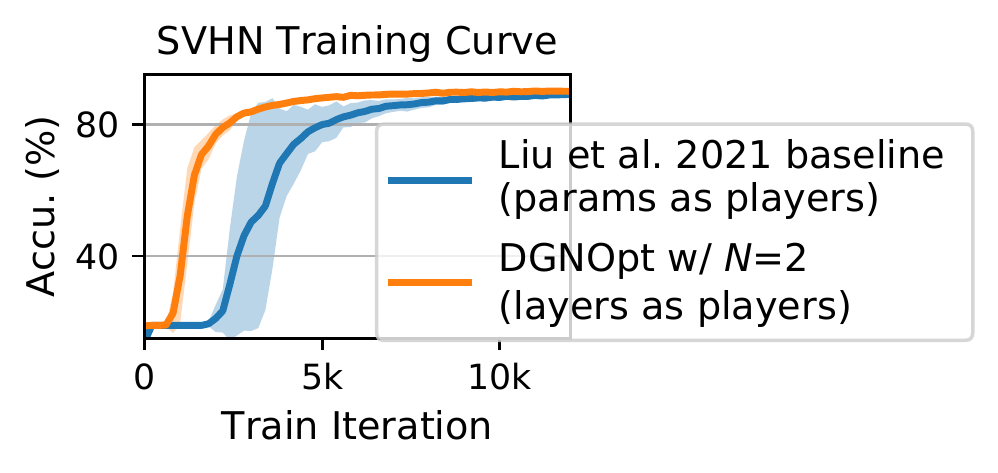}
  \vskip -0.2in
  \caption{
    Convergence using different dynamic game formulations with the same setup as Fig.~\ref{fig:multi-player-a}.
  }
  \label{fig:players-comp}
\end{wrapfigure}
which converges much slower than our DGNOpt (Fig.~\ref{fig:players-comp}).
Our proposed transformation (\ref{eq:robust-traj}) enables collaborative control
and may be naturally extended to other robust formulations, \eg minimax adversarial training.

\endgroup

\section{Conclusion} \label{sec:conclusion}

In this work, we introduce a novel game-theoretic characterization
by bridging the training process of DNN with a multi-agent dynamic game.
The inspired optimizer, DGNOpt,
generalizes previous OCT-based methods to generic network class and encourages cooperative updates to improve the performance.
Our work pushes forward principled algorithmic design from OCT and game theory.

\vspace{-0.5mm}
\section*{Acknowledgements}
\vspace{-0.5mm}
G.H. Liu was supported by CPS NSF Award \#1932068, and
T. Chen was supported by ARO Award \#W911NF2010151.
The authors thank C.H. Lin, Y. Pan, C.W. Kuo, M. Gandhi, E. Evans, and Z. Wang for many helpful discussions.

\newpage
\clearpage

\bibliography{reference}
\bibliographystyle{icml2021}

\appendix
\newpage
\onecolumn
\begin{center}
\textbf{{\huge Supplementary Material}}
\end{center}

\section{Notation Summary} \label{sec:notation}

\bgroup
\def\arraystretch{1.1}%

\begin{table}[H]
\vskip -0.1in
\caption{Abbreviation.}
\vskip 0.1in
\begin{center}
    \begin{small}
    \begin{tabular}{ll}
    \toprule
    OCT/OCP    & Optimal Control Theory/Programming \\ [1pt]
    MPDG & Multi-Player Dynamic Game \\ [1pt]
    CG & Cooperative Game \\ [1pt]
    OLNE   & Open-loop Nash Equilibria   \\ [1pt]
    FNE   & Feedback Nash Equilibria   \\ [1pt]
    GR   & Group Rationality   \\ [1pt]
    IR   & Individual Rationality   \\ [1pt]
    \bottomrule
    \end{tabular}
    \end{small}
\end{center}
\vskip -0.1in
\end{table}

\def\arraystretch{1.3}
\begin{table}[H]
\vskip -0.1in
\caption{Terminology mapping.}
\vskip 0.1in
\begin{center}
    \begin{small}
    \begin{tabular}{lll}
    \toprule
    \multicolumn{1}{c}{} & \textbf{MPDG} & $\text{ }\text{ }$ \textbf{Training generic (non-Markovian) DNNs} \\
    \midrule
    \specialcelll[l]{$t$\\$n$} & \specialcelll[l]{Stage order\\Player index} &
    $\left.\begin{array}{l} \text{Computation order from input to output}\\\text{Index of parallel layers aligned at $t$} \end{array}\right\}$ Layer index $(t,n)$ \\
    $f_{t,n}$ & $\qquad$ - & $\text{ }\text{ }$ Layer module indexed by $(t,n)$ \\
    $F_t$     & Shared dynamics & $\text{ }\text{ }$ Joint propagation rule of layers $\{f_{t,n}: n\in[N] \}$ \\
    $\theta_{t,n}$ & {Action committed at stage $t$ by Player $n$} & $\text{ }\text{ }$ Trainable parameter of layer $f_{t,n}$ \\
    $\vz_{t,n}$      & $\qquad$ -  & $\text{ }\text{ }$ Pre-activation vector of layer $f_{t,n}$ \\
    $\vx_{t}$      & State at stage $t$  & $\text{ }\text{ }$ Augmentation of pre-activation vectors of layers $\{f_{t,n}: n\in[N] \}$ \\
    $\ell_{t,n}$   & {Cost incurred at stage $t$ for Player $n$}  & $\text{ }\text{ }$ Weight decay for layer $f_{t,n}$ \\
    $\phi_{n}$     & {Cost incurred at final stage $T$ for Player $n$} & $\text{ }\text{ }$ {Lost w.r.t. network output (\eg cross entropy in classification)} \\
    \bottomrule
    \end{tabular}
    \end{small}
\end{center}
\vskip -0.1in
\end{table}

\begin{table}[H]
\vskip -0.1in
\caption{Dynamic game theoretic terminology w.r.t. different optimality principles.}
\vskip 0.1in
\begin{center}
    \begin{small}
    \begin{tabular}{c!{\vrule width 0.7pt}ll}
    \toprule
    \multirow{3}{*}{\textbf{OLNE}}
    & $\etaOL$     & Open-loop information structure \\
    & $H_{t,n}$    & Optimality objective (Hamiltonian) for OLNE \\
    & $\vp_{t,n}$  & Co-state at stage $t$ for Player $n$ \\
    \midrule
    \multirow{5}{*}{\textbf{FNE}}
    & $\etaCL$   & Feedback information structure \\
    & $Q_{t,n}$  & Optimality objective (Isaacs-Bellman objective) for FNE \\
    & $V_{t,n}$  & Value function for FNE \\
    & $\ktn$     & Open gain of the locally optimal update for FNE \\
    & $\Ktn$     & Feedback gain of the locally optimal update for FNE \\
    \midrule
    \multirow{6}{*}{\textbf{GR}}
    & $\etaOCG$ & Cooperative open-loop information structure \\
    & $\etaCCG$ & Cooperative feedback information structure \\
    & $P_{t}$   & Optimality objective (group Bellman objective) for GR \\
    & $W_{t}$   & Value function for GR \\
    & $\kut$    & Open gain of the locally optimal update for GR \\
    & $\Kut$    & Feedback gain of the locally optimal update for GR \\
    \bottomrule
    \end{tabular}
    \end{small}
\end{center}
\vskip -0.1in
\end{table}

\egroup

\newpage

\section{OCP Characterization of Training Feedforward Networks} \label{sec:app2}

The optimality principle to OCP (\ref{eq:ocp-obj}), or equivalently the training process of feedforward networks, can be characterized by Dynamic Programming (DP) or Pontryagin Principle (PP). We synthesize the related results below.
\begin{theorem}[\citet{bellman1954theory,pontryagin1962mathematical}] \label{the:pmp-dp} %
  $\text{ }$ \\
  \textbf{(DP)} Define a value function $V_t$ computed recursively by the Bellman equation (\ref{eq:bellman}), starting from $V_T(\vz_T) = \phi(\vz_T)$,
  \begin{equation} \label{eq:bellman}
  \begin{split}
    V_t(\vz_t) &= \min_{\pi_t} {
    \text{ } Q_t(\vz_t,\theta_t), \quad \text{where } \quad
    Q_t(\vz_t,\theta_t) \coloneqq \ell_t(\theta_t) + V_{t+1}(f_t(\vz_t,\theta_t))
    }
  \end{split}
  \end{equation}
  is called the Bellman objective. $\pi_t \equiv \theta_t(\vz_t)$ is an arbitrary mapping from $\vz_t$ to $\theta_t$.
  Let $\pi_t^*$ be the minimizer to (\ref{eq:bellman}),
  then $\{ \pi_t^*: t\in[T] \}$ is the optimal feedback policy to ({\ref{eq:ocp-obj}}).

  \textbf{(PP)} The optimal trajectory
  $\pi_t^* \equiv \theta_t^*(\vz_t^*)$
  along ({\ref{eq:bellman}}) obeys
  \begin{subequations} \label{eq:mf-pmp}
    \begin{alignat}{2}
    {\vz}_{t+1}^{*}&= \nabla_{\vp_{t+1}} H_t\left(\vz_{t}^{*}, \vp_{t+1}^{*}, \theta_{t}^{*}\right),\quad
    \vz_{0}^{*}&&=\vz_{0},
    \label{eq:pmp-forward} \\
    {\vp}_{t}^{*}&=\nabla_{\vz_t} H_t\left(\vz_{t}^{*}, \vp_{t+1}^{*}, \theta_{t}^{*}\right), \quad\quad
    \vp_{T}^{*}&&= \nabla_{\vz_T} \phi\left(\vz_{T}^{*}\right),
    \label{eq:pmp-backward} \\
    \theta_{t}^{*} &= \argmin_{\theta_t}
    H_t\left(\vz_{t}^{*}, \vp_{t+1}^{*}, \theta_t \right),
    \label{eq:pmp-max-h}
    \end{alignat}
  \end{subequations}
  where ({\ref{eq:pmp-backward}}) is the {adjoint equation} for the co-state $\vp^*_t$ and
  \begin{align*}
  H_t\left(\vz_t, \vp_{t+1}, \theta_t \right) \coloneqq \ell_t(\theta_t) + \vp_{t+1}^\transpose f_t(\vz_t, \theta_t)
  \end{align*}
  is the discrete-time Hamiltonian.
\end{theorem}

Theorem \ref{the:pmp-dp} provides an OCP characterization of training feedforward networks.
First, notice that
the time-varying OCP objectives $(Q_t,H_t)$
are constructed through some backward processes similar to the Back-propagation (BP).
Indeed, one can verify that the adjoint equation (\ref{eq:pmp-backward}) gives the {exact} BP dynamics. %
Similarly, the dynamics of $V_t$ in (\ref{eq:bellman}) also relate to BP under some conditions \citep{liu2021differential}.
The parameter update, $\theta_t \leftarrow \theta_t - \delta \theta_t$, for standard training methods can be seen as solving the discrete-time Hamiltonian $H_t$
with different precondition matrices \citep{li2017maximum}.
On the other hand, DDPNOpt \citep{liu2021differential} minimizes the time-dependent Bellman objective $Q_t$ with
$\theta_t \leftarrow \theta_t - \delta \pi_t$.
This elegant connection
is, however, limited to the interpretation between feedforward networks and Markovian dynamical systems (\ref{eq:layer-prop}).

\section{Missing Derivations in Section \ref{sec:3}} \label{sec:app3}

\begin{figure}[H]
\begin{center}
\includegraphics[width=0.85\linewidth]{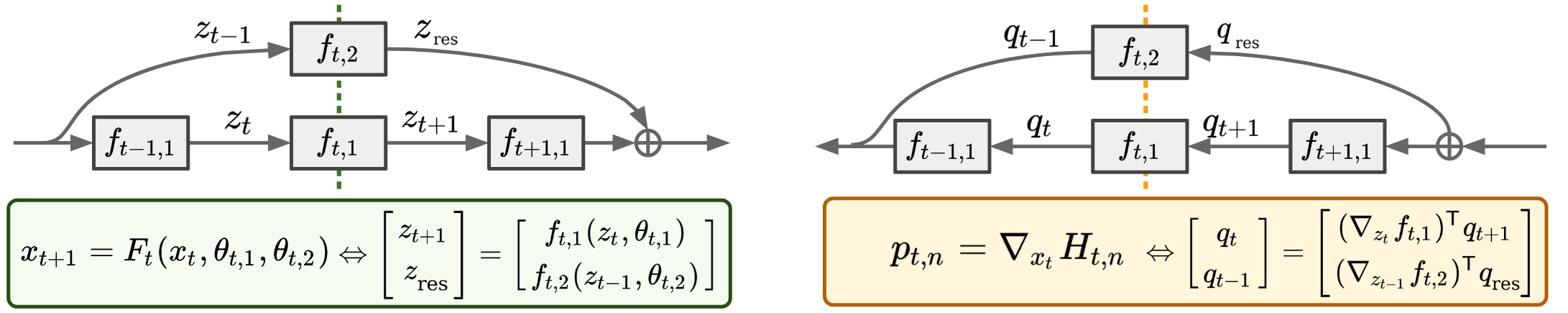}
\caption{ Forward propagation (left) and Back-propagation (right) of a residual block and how each quantity connects to OLNE optimality.
}
\label{fig:app1}
\end{center}
\vskip -0.2in
\end{figure}

\textbf{Proof of Proposition~\ref{prop:nash-open-loop}.}
Expand the expression of the Hamiltonian in OLNE:
\begin{align*}
    H_{t,n}(\vx_t, \vp_{t+1,n}, \theta_{t,1}, \cdots, \theta_{t,N})
    \coloneqq& \text{ } \ell_{t,n}(\theta_{t,1}, \cdots, \theta_{t,N}) + F_t(\vx_{t}, \theta_{t,1}, \cdots, \theta_{t,N})^\transpose \vp_{t+1,n},
\end{align*}
where $\vp_{t,n}$ is the co-state whose dynamics obey
\begin{align*}
    \vp_{t,n} = \nabla_{\vx_t} H_{t,n}, \quad
    \vp_{T,n}  = \nabla_{\vx_T} &\phi_n(\vx_T).
\end{align*}
Recall $\S$\ref{sec:dnn-dg}
where we demonstrate that for training generic DNNs,
one shall consider $\ell_{t,n} \coloneqq \ell_{t,n}(\theta_{t,n})$ and $\phi_n \coloneqq \phi$.
Hence, the dynamics of $\vp_{t,n}$ become
\begin{align}
    \vp_{t,n} = \nabla_{\vx_t} H_{t,n}, \quad
    \vp_{T,n} = \nabla_{\vx_T} &\phi(\vx_T),
    \quad \text{ where } H_{t,n} = \ell_{t,n}(\theta_{t,n}) + F_t(\vx_{t}, \theta_{t,1}, \cdots, \theta_{t,N})^\transpose  \vp_{t+1,n}.
    \label{eq:19}
\end{align}
Our goal is to show that (\ref{eq:19}) gives the exact Back-propagation dynamics.
First, notice that the terminal condition of (\ref{eq:19}), \ie $\vp_{T,n} = \nabla_{\vx_T} \phi$,
is already the gradient w.r.t. the network output without any manipulation.
Next, to show that $\vp_{t,n}$ corresponds to the Back-propagation at stage $t$,
consider, for instance, the computation graphs of the residual block in Fig.~\ref{fig:app1},
where we replot Fig.~\ref{fig:dy-net} together with its Back-propagation dynamic
and denote $\vq$ as the gradient w.r.t. the activation vector $\vz$.
Then, it can be shown by induction that $\vp_{t,n}$ augments all ``$\vq$''s aligned at stage $t$.
Indeed,
suppose $\vp_{t+1,n}$ is the augmentation of the Back-propagation gradients at stage $t{+}1$,
\ie $\vp_{t+1,n} \coloneqq [\vq_{t+1},\vq_{\text{res}}]^\transpose$,
then the co-state at the current stage $t$ can be expanded as
\begin{align*}
    \vp_{t,n} = \nabla_{\vx_t} H_{t,n} = \nabla_{\vx_t} F_t^\transpose \vp_{t+1,n}
    =
    \left[
    \begin{array}{cc}
        \nabla_{\vz_t}f_{t,1}   & \nabla_{\vz_{t-1}}f_{t,1} \\
        \nabla_{\vz_{t}}f_{t,2} & \nabla_{\vz_{t-1}}f_{t,2}
    \end{array}
    \right]^\transpose
    \left[
    \begin{array}{c}{\vq_{t+1}} \\ {\vq_{\text{res}} }\end{array}
    \right]
    =
    \left[
    \begin{array}{c}{\nabla_{\vz_t} {f_{t,1}}^\transpose \vq_{t+1}}\\{\nabla_{\vz_{t-1}} {f_{t,2}}^\transpose \vq_{\text{res}}}\end{array}
    \right]
    =
    \left[
    \begin{array}{c}{\vq_{t}}\\{\vq_{t-1}}\end{array}
    \right]
    ,
\end{align*}
which augments all ``$\vq$''s at stage $t$.
Once we connect $\vp_{t,n}$ to the Back-propagation dynamics, it can be verified that
\begin{align*}
    H^{t,n}_{\theta} \equiv \nabla_{\theta_{t,n}} H_{t,n} = \nabla_{\theta_{t,n}} \ell_{t,n} + \nabla_{\theta_{t,n}} F_t^\transpose \vp_{t+1,n}.
\end{align*}
is indeed the gradient w.r.t. the parameter $\theta_{t,n}$ of each layer $f_{t,n}$.
Therefore, taking the iterative update $\theta_{t,n}\leftarrow\theta_{t,n} - H^{t,n}_{\theta}$ is equivalent to descending along the SGD direction, up to a learning rate scaling.
Similarly, setting different precondition matrices $\mM$ will recover other standard methods.
Hence, we conclude the proof.

\hfill $\qedsymbol$

\textbf{Optimality principle for \textnormal{$\etaOCG$}.}
For the completeness, below we provide the optimality principle for the cooperative open-loop information structure $\etaOCG$.
\begin{definition}[Cooperative optimality principle by $\etaOCG$]
    A set of strategy,
    $\{ \theta^*_{t,n} : \forall t,n \}$,
    provides an open-loop optimal solution to the joint optimization (\ref{eq:group-ration}) if
    \begin{align*}
        \theta^*_{t,1}, {\cdots}, \theta^*_{t,N} =& \argmin_{\theta_{t,n}:n\in[N]} \text{ } \widetilde{H}_{t}
            (\vx_t,\widetilde{\vp}_{t+1},\theta_{t,1}, {\cdots}, \theta_{t,N}),
        \\
         \text{ where} \quad
         \theta^*_{t,n} \equiv \theta^*_{t,n}(\etaOCG)
         \quad &\text{and} \quad
        \widetilde{H}_{t} \coloneqq \textstyle\sum_{n=1}^N \ell_{t,n} + F_t^\transpose \widetilde{\vp}_{t+1}
    \end{align*}
    is the ``group'' Hamiltonian at stage $t$. Similar to OLNE, the joint co-state $\widetilde{\vp}_{t}$ can be simulated by
    \begin{align*}
        \widetilde{\vp}_{t} = \nabla_{\vx_t} \widetilde{H}_{t}, \quad \widetilde{\vp}_{T} = \textstyle\sum_{n=1}^N \nabla_{\vx_T} \phi_{n}.
    \end{align*}
\end{definition}
In this work, we focus on solving the optimality principle inherited in $\etaCCG$ as a representative of the CG optimality.
Since $\etaOCG \subset \etaCCG$, the latter captures richer information and tends to perform better in practice,
as evidenced by Fig.~\ref{fig:mnist}.

\section{Missing Derivations in Section \ref{sec:4}} \label{sec:app4}

\subsection{Complete Derivation of the Iterative Updates} \label{sec:C1}

\textbf{Derivation of FNE update.}
Our goal is to approximately solve the Isaacs-Bellman recursion (\ref{eq:i-bellman}) only up to second-order.
Recall that the second-order expansion of $Q_{t,n}$ at some fixed point $(\vx_t,\theta_{t,n})$
takes the form
\bgroup
\def\arraystretch{1.3}%
\begin{align}
        Q_{t,n} \approx
        \frac{1}{2}
        \left[\begin{array}{c} {\mathbf{1}} \\ {\dvx_t} \\ {\dth_{t,n}} \end{array}\right]^{\transpose}
        \left[\begin{array}{lll}
            {Q_{t,n}}    & {\Qxt^\transpose} & {\Qut^\transpose} \\
            {\Qxt} & {\Qxxt} & {\Quxt^\transpose} \\
            {\Qut} & {\Quxt} & {\Quut}
            \end{array}\right]
        \left[\begin{array}{c} {\mathbf{1}} \\ {\dvx_t} \\ {\dth_{t,n}} \end{array}\right]
    , \quad \text{where }
    \begin{array}{r@{=}}
    {\Qxt } \\
    {\Qut } \\
    {\Quut } \\
    {\Quxt } \\
    {\Qxxt }
    \end{array}
    \begin{array}{l@{}}
    {\FxtT \Vxtn} \\
    {\FutT \Vxtn + \lutn} \\
    {\FutT \Vxxtn \text{ } \Fut + \luutn} \\
    {\FutT \Vxxtn \text{ } \Fxt} \\
    {\FxtT \Vxxtn \text{ } \Fxt }
    \end{array}
    \label{eq:fne-expand}
\end{align}
\egroup
follow standard chain rule (recall $Q_{t,n}\coloneqq \ell_{t,n} + V_{t+1,n} \circ F_t$) with the linearized dynamics
$\Fut \equiv \nabla_{\theta_{t,n}} F_t$ and $\Fxt \equiv \nabla_{\vx_{t}} F_t$.
The expansion (\ref{eq:fne-expand}) is a standard quadratic programming, and its analytic solution is given by
\begin{align*}
    -\delta \pi_{t,n}^* \equiv - \delta \theta_{t,n}^*(\dvx_t) =  - (\Quutn)^\Inv (\Qutn + \Quxtn \dvx_t) \eqqcolon - (\ktn + \Ktn \dvx_t).
\end{align*}
Substituting this solution back to the Isaacs-Bellman recursion gives us the local expression of $V_{t,n}$,
\begin{align}
    V_{t,n} \approx Q_{t,n} - \frac{1}{2}  (\Qutn)^\transpose (\Quutn)^\Inv \Qutn.
    \label{eq:ibeq}
\end{align}
Therefore, the local derivatives of $V_{t,n}$ can be computed by
\begin{align*}
    V_\vx^{t,n} &= Q_\vx^{t,n} -  Q_{\vx\theta}^{t,n} (\Quutn)^\Inv \Qutn = Q_\vx^{t,n} - Q_{\vx\theta}^{t,n} \ktn \\
    V_{\vx\vx}^{t,n} &= Q_{\vx\vx}^{t,n} -  Q_{\vx\theta}^{t,n} (\Quutn)^\Inv Q_{\theta\vx}^{t,n} = Q_{\vx\vx}^{t,n} - Q_{\vx\theta}^{t,n} \Ktn.
\end{align*}

\newcommand\bovermat[2]{%
  \makebox[0pt][l]{$\smash{\overbrace{\phantom{%
    \begin{matrix}#2\end{matrix}}}^{\text{#1}}}$}#2}

\newcommand\bundermat[2]{%
  \makebox[0pt][l]{$\smash{\underbrace{\phantom{%
    \begin{matrix}#2\end{matrix}}}_{\text{#1}}}$}#2}

\textbf{Derivation of GR update.}
We will adopt the same terminology $\vu \equiv \theta_{t,1}, \vv \equiv \theta_{t,2}$.
Following the procedure as in the FNE case, we can perform the second-order expansion of $P_t$ at some fixed point $(\vx_t,\vu,\vv)$.
The analytic solution to the corresponding quadratic programming is given by
\begin{align}
    -\left[\begin{array}{c}
        {\delta \pi^*_{t,1}} \\
        {\delta \pi^*_{t,2}}
    \end{array}\right]
    = -
    \left[\begin{array}{cc}
        {\Puut} & {\Puvt} \\
        {\Pvut} & {\Pvvt}
    \end{array}\right]^\Inv
    \left(
    \left[\begin{array}{c}
        {\Put} \\
        {\Pvt}
    \end{array}\right]
    +
    \left[\begin{array}{c}
        {\Puxt} \\
        {\Pvxt}
    \end{array}\right]
    \dvx_t\right),
    \label{eq:coop-du-dev}
\end{align}
where the block-matrices inversion can be expanded using the Schur complement.

\begin{align}
\begin{bmatrix}
    {\Puut} & {\Puvt} \\
    {\Pvut} & {\Pvvt}
\end{bmatrix}^{\Inv}
=
\begin{bmatrix}
    (\bovermat{$\PuuCt$}{\Puut-\Puvt\PvvInvt\Pvut})^\Inv & {-\PuuCInvt\Puvt\PvvInvt} \\
    {-\PvvCInvt\Pvut\PuuInvt} & (\bundermat{$\PvvCt$}{\Pvvt-\Pvut\PuuInvt\Puvt})^\Inv
\end{bmatrix}
\period \label{eq:schur}
\end{align}

Hence, (\ref{eq:coop-du-dev}) becomes
\begin{align*}
    \left[\begin{array}{c}
        {\delta \pi^*_{t,1}} \\
        {\delta \pi^*_{t,2}}
    \end{array}\right]
    &=
    \left[\begin{array}{c}
        {\PuuCInvt(\Put - \Puvt\PvvInvt\Pvt)} \\
        {\PvvCInvt(\Pvt - \Pvut\PuuInvt\Put)}
    \end{array}\right]
    +
    \left[\begin{array}{c}
        {\PuuCInvt(\Puxt - \Puvt\PvvInvt\Pvxt)} \\
        {\PvvCInvt(\Pvxt - \Pvut\PuuInvt\Puxt)}
    \end{array}\right]
    \dvx_t \\
    &=
    \left[\begin{array}{c}
        {\PuuCInvt(\Put - \Puvt\It)} \\
        {\PvvCInvt(\Pvt - \Pvut\kt)}
    \end{array}\right]
    +
    \left[\begin{array}{c}
        {\PuuCInvt(\Puxt - \Puvt\Lt)} \\
        {\PvvCInvt(\Pvxt - \Pvut\Kt)}
    \end{array}\right]
    \dvx_t \\
    &\eqqcolon
    \left[\begin{array}{c}
        {\kut} \\
        {\Ivt}
    \end{array}\right]
    +
    \left[\begin{array}{c}
        {\Kut} \\
        {\Lvt}
    \end{array}\right] \dvx_t ,
\end{align*}
where we denote the non-cooperative iterative update for Player 1 and 2 respectively by
\begin{align*}
    \delta \vu_t(\dvx_t) &= \kt + \Kt\dvx_t,
        \quad \text{where }\text{ } \kt \coloneqq \PuuInvt\Put \quad
        \text{ and } \quad \Kt \coloneqq \PuuInvt\Puxt, \\
    \delta \vv_t(\dvx_t) &= \It + \Lt\dvx_t,
        \quad \text{ }\text{ } \text{where }\text{ } \It \coloneqq \PvvInvt\Pvt  \quad
        \text{ } \text{ and } \quad \Lt \coloneqq \PvvInvt\Pvxt.
\end{align*}
Substituting this solution back to the GR Bellman equation gives the local expression of $W_{t}$,
\begin{align}
    W_{t} \approx P_{t} - \frac{1}{2}
    \left[\begin{array}{c}
        {\Put} \\
        {\Pvt}
    \end{array}\right]^\transpose
    \begin{bmatrix}
        {\Puut} & {\Puvt} \\
        {\Pvut} & {\Pvvt}
    \end{bmatrix}^{\Inv}
    \left[\begin{array}{c}
        {\Put} \\
        {\Pvt}
    \end{array}\right].
    \label{eq:grb}
\end{align}
Finally, taking the derivatives yields the formula for updating the derivatives of $W_{t}$,
\begin{align}
    W_\vx^{t} = P_\vx^{t} - \frac{1}{2} \left( \Pxut\kut + \Pxvt\Ivt + \KutT\Put + \LvtT\Pvt \right)
    \quad \text{and} \quad
    W_{\vx\vx}^{t} = P_{\vx\vx}^{t,n} - \Pxut\Kut - \Pxvt\Lvt,
    \label{eq:gr-vxx}
\end{align}
which is much complex than (\ref{eq:fne-vxx}).

\subsection{Kronecker Factorization and Proof of Proposition~\ref{prop:coop}} \label{sec:C2}
We first provide some backgrounds for the Kronecker factorization (KFAC; \citet{martens2015optimizing}).
KFAC relies on the fact that for an affine mapping layer, \ie
$\vz_{t+1} = f_t(\vz_t,\theta_t) \coloneqq \mW_t \vz_t + \vb_t, \text{ } \theta_t \coloneqq \vectorize([\mW_t, \vb_t])$,
the gradient of the training objective $L$ w.r.t. the parameter $\theta_t$ admits a compact factorization,
\begin{align*}
    \nabla_{\theta_t} L = \nabla_{\theta_t} f_t^\transpose \nabla_{\vz_{t+1}} L = \vz_t \otimes \nabla_{\vz_{t+1}} L,
\end{align*}
where $\otimes$ denotes the Kronecker product.
With this, the layer-wise Fisher information matrix, or equivalently the Gauss-Newton (GN) matrix, for classification can be approximated with
\begin{align*}
    \E[\nabla_{\theta_t} L \nabla_{\theta_t} L^\transpose]
    = \E[(\vz_t \otimes \nabla_{\vz_{t+1}} L) (\vz_t\otimes \nabla_{\vz_{t+1}} L)^\transpose]
    \approx \E[\vz_t\vz_t^\transpose] \otimes \E[\nabla_{\vz_{t+1}} L \nabla_{\vz_{t+1}} L^\transpose].
\end{align*}
We can adopt this factorization to our setup by first recalling from
our proof of Proposition~\ref{prop:nash-open-loop} (see Appendix~\ref{sec:app3})
that
$(\nabla_{\theta_t} L, \nabla_{\vz_{t+1}} L)$ are interchangeable with
$(H_{\theta}^t, \vp_{t+1})$, or equivalently $(H_{\theta}^t, H_{\vz}^{t+1})$.
Hence, the GN approximation of $\E[H^{t}_{\theta\theta}]$ can be factorized by
\begin{align}
    \E[\Hut\Hut^\transpose]
    \approx \E[\vz_t\vz_t^\transpose] \otimes \E[\vp_{t+1} \vp_{t+1}^\transpose]
    = \E[\vz_t\vz_t^\transpose] \otimes \E[H_{\vz}^{t+1} H_{\vz}^{t+1\text{ }\transpose}].
    \label{eq:gn-pmp}
\end{align}
Equation (\ref{eq:gn-pmp}) suggests that
KFAC factorizes the parameter curvature with two smaller matrices using the {activation} state $\vz_t$
and the derivative of \emph{some optimality} (in this case the Hamiltonian $H$) w.r.t. $\vz_{t+1}$.
The main advantage of this factorization is to exploit the following formula,
\begin{align}
      (\mA \otimes \mB)^\Inv \vectorize (\mW)
    = (\mA^\Inv \otimes \mB^\Inv) \vectorize (\mW)
    = \vectorize (\mB^\Inv \mW \mA^{-\transpose}),
    \label{eq:kfac-formula}
\end{align}
which allows one to efficiently inverse the parameter curvature with two smaller matrices.

Now, let us proceed to the proof of Proposition~\ref{prop:coop}.
First notice that for the shared dynamics considered in Fig.~\ref{fig:dy-net}, we have
\begin{align*}
    F_t(\vx_t,\vu,\vv) \coloneqq
    \left[\begin{array}{c}
        {f_{t,1}(\vz_1,\vu)} \\
        {f_{t,2}(\vz_2,\vv)}
    \end{array}\right]
    =
    \left[\begin{array}{cc}
        {f_{t,1}(\cdot,\vu)} & {\mathbf{0}} \\
        {\mathbf{0}}         & {f_{t,2}(\cdot,\vv)}
    \end{array}\right]
    \left[\begin{array}{c}
        {\vz_1} \\
        {\vz_2}
    \end{array}\right],
\end{align*}
which resembles the affine mapping concerned by KFAC.
This motivates the following approximation,
\begin{align}
    \E[ P_{\theta}^t {P_{\theta}^t}^\transpose]
    \approx \E[\vx_t\vx_t^\transpose] \otimes \E[W_{\vx}^{t+1} {W_{\vx}^{t+1}}^\transpose].
    \label{eq:coop-kfac2}
\end{align}
Similar to (\ref{eq:gn-pmp}), this approximation (\ref{eq:coop-kfac2}) factorizes the GN matrix with the MPDG state $\vx_t$
and the derivative of an optimality (in this case it becomes the GR value function $W_{t+1}$) w.r.t. $\vx_{t+1}$.

If we denote the derivatives w.r.t. the outputs of $f_{t,1}$ and $f_{t,2}$ by $\vg_1$ and $\vg_2$,
\ie $W_{\vx}^{t+1} \coloneqq [\vg_1,\vg_2]^\transpose$,
and rewrite $\vx_t \coloneqq [\vz_1,\vz_2]^\transpose$,
then (\ref{eq:coop-kfac2}) can be expanded by
\begin{align*}
    \E[{\vx_t} {\vx_t}^\transpose]
    =&
        \begin{bmatrix}
            \E[{\vz_1} {\vz_1}^\transpose] & \E[{\vz_1} {\vz_2}^\transpose] \\
            \E[{\vz_2} {\vz_1}^\transpose] & \E[{\vz_2} {\vz_2}^\transpose]
        \end{bmatrix}
    \eqqcolon
        \begin{bmatrix}
            \Auu & \Auv \\
            \Avu & \Avv
        \end{bmatrix}
    \\
    \E[{W_{\vx}^{t+1}} {W_{\vx}^{t+1}}^\transpose]
    =&
        \begin{bmatrix}
            \E[{\vg_1} {\vg_1}^\transpose] & \E[{\vg_1} {\vg_2}^\transpose] \\
            \E[{\vg_2} {\vg_1}^\transpose] & \E[{\vg_2} {\vg_2}^\transpose]
        \end{bmatrix}
    \eqqcolon
        \begin{bmatrix}
            \Buu & \Buv \\
            \Bvu & \Bvv
        \end{bmatrix}
    .
\end{align*}
Their inverse matrices are given by the Schur component.
\begin{equation}
    \begin{split}
    \begin{bmatrix}
        \Auu & \Auv \\
        \Avu & \Avv
    \end{bmatrix}^\Inv
    &=
        \begin{bmatrix}
            \AuuCInv & - \AuuCInv \Auv \AvvInv \\
            - \AvvCInv \Avu \AuuInv & \AvvCInv
        \end{bmatrix}
    \comma \quad \text{where}
    \begin{cases}
        \AuuC \coloneqq \Auu - \Auv \AvvInv \AuvT \\
        \AvvC \coloneqq \Avv - \Avu \AuuInv \AvuT
    \end{cases}
    \\
    \begin{bmatrix}
    \Buu & \Buv \\
    \Bvu & \Bvv
    \end{bmatrix}^\Inv
    &=
        \begin{bmatrix}
            \BuuCInv & - \BuuCInv \Buv \BvvInv \\
            - \BvvCInv \Bvu \BuuInv & \BvvCInv
        \end{bmatrix}
    \comma \quad \text{where}
    \begin{cases}
        \BuuC \coloneqq \Buu - \Buv \BvvInv \BuvT \\
        \BvvC \coloneqq \Bvv - \Bvu \BuuInv \BvuT
    \end{cases}
\end{split} \label{eq:ABcoop}
\end{equation}
With all these, the cooperative open gain can be computed with the formula (\ref{eq:kfac-formula}),
\begin{equation}
\begin{split}
    \left(\E[\vx_t\vx_t^\transpose] \otimes \E[W_{\vx}^{t+1} {W_{\vx}^{t+1}}^\transpose]\right)^\Inv
    \vectorize(\begin{bmatrix} \Put & \mathbf{0} \\ \mathbf{0} & \Pvt \end{bmatrix})
=   \vectorize\left(
    \begin{bmatrix}
        \Buu & \Buv \\
        \Bvu & \Bvv
    \end{bmatrix}^\Inv
    \begin{bmatrix} \Put & \mathbf{0} \\ \mathbf{0} & \Pvt \end{bmatrix}
    \begin{bmatrix}
        \Auu & \Auv \\
        \Avu & \Avv
    \end{bmatrix}^{-\transpose}\right).
\end{split} \label{eq:46}
\end{equation}
Substituting (\ref{eq:ABcoop}) into (\ref{eq:46}),
after some algebra we will arrive at \textit{the KFAC of the cooperative open gain} suggested in (\ref{eq:coop-open}).
\begin{align*}
    \kut
\approx&  \vectorize(\markblue{\BuuCInv} \Put \markblue{\AuuCInvT} + \markblue{\BuuCInv} \Buv \BvvInv \Pvt (\markblue{\AuuCInv} \Auv \AvvInv )^\transpose) \\
=&  \vectorize(\markblue{\BuuCInv} (\Put + \Buv \markgreen{\BvvInv \Pvt \AvvInvT} \AuvT) \markblue{\AuuCInvT}) \\
=&  \vectorize(\BuuCInv (\Put + \Buv \markgreen{\vectorize^\Inv(\It)} \AuvT) \AuuCInvT) \label{eq:28} \numberthis,
\end{align*}
where the last equality follows by another KFAC approximation 
$\markgreen{\It \approx (\Avv \otimes \Bvv)^\Inv\vectorize(\Pvt) = \vectorize(\BvvInv \Pvt \AvvInvT)}$.
Finally, recalling the expression, $\kut \coloneqq  \PuuCInvt (\Put  - \Puvt \It)$, from (\ref{eq:gr-update})
and rewriting (\ref{eq:28}) by
\begin{align*}
    \kut
\approx&  \vectorize(\BuuCInv (\Put + \Buv {\vectorize^\Inv(\It)} \AuvT) \AuuCInvT) \\
=&  (\AuuC \otimes \BuuC)^\Inv \vectorize(\Put + \Buv \vectorize^\Inv(\It) \AuvT)
\end{align*}
imply the KFAC representation $\PuuCt \approx \AuuC \otimes \BuuC$ in (\ref{eq:coop-kfac}).
Hence we conclude the proof.
\hfill $\qedsymbol$

\subsection{Proof of Theorem~\ref{thm:alg-connection}} \label{sec:C3}
We first show that setting $\Puvt\coloneqq\mathbf{0}$ in the update (\ref{eq:gr-update}) yields (\ref{eq:fne-update}).
To begin, observe that when $\Puvt$ vanishes, the cooperative gains $(\kut,\Kut)$ appearing in (\ref{eq:gr-update}) degenerate to
$\kut = \PuuInvt\Put$ and $\Kut = \PuuInvt \Puxt$.
Therefore, it is sufficient to prove the following result.\footnote{
  We consider the two-player setup for simplicity, yet the methodology applies to the multi-player setup.
}
\begin{lemma} \label{lemma9}
    Suppose $Q_{t,n}$ in (\ref{eq:i-bellman}) and $P_t$ in (\ref{eq:gr-bellman})
    are expanded up to second-order along the same local trajectory
    $\{(\vx_t,\theta_{t,n},\cdots,\theta_{t,N}): \forall t \in [T]\}$,
    then we will have the following relations when $\Puvt\coloneqq\mathbf{0}$ at all stages.
    \begin{align}
    \begin{split}
        \forall t, \quad
        Q^{t,1}_{\theta} = \Put, \quad
        Q^{t,1}_{\theta\theta} = \Puut, \quad
        Q^{t,1}_{\theta\vx} = \Puxt, \quad
        Q^{t,2}_{\theta} = \Pvt, \quad
        Q^{t,2}_{\theta\theta} = \Pvvt, \quad
        Q^{t,2}_{\theta\vx} = \Pvxt,
        \label{eq:lemma9}
    \end{split}
    \end{align}
    where $(\vu_t,\vv_t) \equiv (\theta_{t,1},\theta_{t,2})$ denotes the actions for Player 1 and 2.
    Furthermore, we have
    \begin{align}
        \forall t, \quad W_t = \textstyle\sum_{n=1}^N V_{t,n}.
        \label{eq:lemma99}
    \end{align}
\end{lemma}
\begin{proof}
    We will proceed the proof by induction. At the terminal stage $T-1$, we have
    \begin{align*}
        P_{T-1} = \sum_{n=1}^2 \ell_{T-1,n} + W_{T} \circ F_{T-1}
                = \sum_{n=1}^2 \left( \ell_{T-1,n} + \phi_{n} \circ F_{T-1} \right)
                = \sum_{n=1}^2 Q_{T-1,n},
    \end{align*}
    since $\phi_n = V_{T,n}$.
    This implies that when solving the second-order expansion for
    $\pi_{T-1,1}$ and $\pi_{T-1,2}$,
    we will have
    \begin{align*}
        \min_{\pi_{T-1,1},\pi_{T-1,2}} P_{T-1} = \min_{\pi_{T-1,1}} Q_{T-1,1} + \min_{\pi_{T-1,2}} Q_{T-1,2}
    \end{align*}
    since the cross-correlation matrix $P_{\vu\vv}^{T-1}$ is discarded.
    Therefore, all equalities in (\ref{eq:lemma9}) hold at this stage.
    Furthermore,
    substituting $P_{\vu\vv}^{T-1} \coloneqq \mathbf{0}$ into (\ref{eq:grb}) yields the following GR value function
    \begin{align*}
        W_{T-1} &= P_{T-1} - \frac{1}{2}\left(({P_{\vu}^{T-1}})^\transpose (P_{\vu\vu}^{T-1})^\Inv {P_{\vu}^{T-1}} + ({P_{\vv}^{T-1}})^\transpose (P_{\vv\vv}^{T-1})^\Inv {P_{\vv}^{T-1}} \right) \\
                &= \sum_{n=1}^2 \left(Q_{T-1,n} - \frac{1}{2} (Q_{\theta}^{T-1,n})^\transpose (Q_{\theta\theta}^{T-1,n})^\Inv Q_{\theta}^{T-1,n}\right)
                 = \sum_{n=1}^2 V_{T-1,n}.
    \end{align*}
    So (\ref{eq:lemma99}) also holds.
    Now, suppose (\ref{eq:lemma9}, \ref{eq:lemma99}) hold at $t+1$, then
    \begin{align*}
        P_t = \sum_{n=1}^2 \ell_{t,n} + W_{t+1} \circ F_{t}
            = \sum_{n=1}^2 \left( \ell_{t,n} + V_{t+1,n} \circ F_{t} \right)
            = \sum_{n=1}^2 Q_{t,n}.
    \end{align*}
    Together with $P_{\vu\vv}^{t} \coloneqq \mathbf{0}$, we can see that all equalities in (\ref{eq:lemma9}) hold.
    Furthermore, it implies that
    \begin{align*}
        W_{t} = P_{t} - \frac{1}{2}\left( \Put^\transpose \PuuInvt \Put + \Pvt^\transpose \PvvInvt \Pvt \right)
              = \sum_{n=1}^2 \left(Q_{t,n} - \frac{1}{2}  (\Qutn)^\transpose (\Quutn)^\Inv \Qutn \right)
              = \sum_{n=1}^2 V_{t,n}.
    \end{align*}
    Hence we conclude the proof.
\end{proof}

Next, we proceed to the second case, which suggests that running {(\ref{eq:fne-update}) with $(\Quxtn,\Quutn)\coloneqq(\mathbf{0},\mI)$ yields SGD}.
Since the FNE update in this case degenerates to $\delta \pi^*_{t,n} = \Qutn$, it is sufficient to prove the following lemma.
\begin{lemma} \label{lemma10}
    Suppose $H_{t,n}$ in (\ref{eq:olne}) and $Q_{t,n}$ in (\ref{eq:i-bellman})
    are expanded up to second-order along the same local trajectory
    $\{(\vx_t,\theta_{t,n},\cdots,\theta_{t,N}): \forall t \in [T]\}$,
    then we will have the following relations when $(\Quxtn,\Quutn)\coloneqq(\mathbf{0},\mI)$ for all stages.
    \begin{align}
        \forall t, \quad
        \Qutn = H_{\theta}^{t,n}, \quad
        V_{\vx}^{t,n} = \vp_{t,n}.
        \label{eq:lemma10}
    \end{align}
\end{lemma}
\begin{proof}
    Again, we will proceed the proof by induction.
    First, notice that $V^{T,n}_\vx = \phi_\vx^n = \vp_{T,n}$ no matter
    whether or not $\Quxtn$ and $\Quutn$ degenerate.
    At the terminal stage $T-1$, we have
    \begin{align*}
        Q_{\theta}^{T-1,n}
        = \ell_{\theta}^{T-1,n} + (F_{\theta}^{T-1})^\transpose V_{\vx}^{T,n}
        = \ell_{\theta}^{T-1,n} + (F_{\theta}^{T-1})^\transpose \vp_{T,n} = H_{\theta}^{T-1,n}.
    \end{align*}
    Also, when $Q_{\theta\vx}^{T-1,n} \coloneqq \mathbf{0}$, (\ref{eq:fne-vxx}) becomes
    \begin{align*}
        V_\vx^{T-1,n}
        = Q_\vx^{T-1,n}
        = (F_{\vx}^{T-1})^\transpose V_{\vx}^{T,n}
        = (F_{\vx}^{T-1})^\transpose \vp_{T,n}
        = H_{\vx}^{T-1,n}
        = \vp_{T-1,n}.
    \end{align*}
    Hence, (\ref{eq:lemma10}) holds at $T-1$. Now, suppose these relations hold at $t+1$, then
    \begin{align*}
        Q_{\theta}^{t,n}
        = \ell_{\theta}^{t,n} + (F_{\theta}^{t})^\transpose V_{\vx}^{t+1,n}
        = \ell_{\theta}^{t,n} + (F_{\theta}^{t})^\transpose \vp_{t+1,n} = H_{\theta}^{t,n}
    \end{align*}
    and similarly
    \begin{align*}
        V_\vx^{t,n}
        = Q_\vx^{t,n}
        = (F_{\vx}^{t})^\transpose V_{\vx}^{t+1,n}
        = (F_{\vx}^{t})^\transpose \vp_{t+1,n}
        = H_{\vx}^{t,n}
        = \vp_{t,n}.
    \end{align*}
    Hence, we conclude the proof.
\end{proof}
Finally, the last case follows readily by combining Lemma \ref{lemma9} and \ref{lemma10},
so we conclude all proofs.
\hfill $\qedsymbol$

\section{More on the Experiments} \label{sec:app5}

All experiments are run with Pytorch on the GPU machines, including GTX 1080 TI, GTX 2070, and TITAN RTX.
We preprocessed all datasets with standardization.
We also perform data augmentation when training CIFAR100.
Below we detail the setup for each experiment.

\begingroup
\setlength{\columnsep}{1pt}%
\begin{wrapfigure}[7]{r}{0.37\textwidth}
  \vspace{-12pt}
   \vskip -0.1in
  \captionof{table}{Hyper-parameter search in Table \ref{table:clf}}
  \vskip 0.1in
  \begin{center}
  \begin{small}
  \begin{tabular}{r|c}
  \toprule
  Standard Baselines & Learning Rate (LR) \\
  \midrule
  SGD             & $(7\mathrm{e}\text{-}2,5\mathrm{e}\text{-}1)$ \\
  Adam \& RMSprop & $(7\mathrm{e}\text{-}4,1\mathrm{e}\text{-}2)$ \\
  EKFAC           & $(1\mathrm{e}\text{-}2,3\mathrm{e}\text{-}1)$ \\
  \bottomrule
  \end{tabular} \label{table:hyper}
  \end{small}
  \end{center}
  \vskip 0.5in
\end{wrapfigure}
  \textbf{Classification (Table~\ref{table:clf} and \ref{table:clf2}).}
  For CIFAR10 and CIFAR100, we use standard implementation of ResNet18 from \url{https://pytorch.org/hub/pytorch_vision_resnet/}.
  As for SVHN and MNIST, the residual network consists of 3 residual blocks.
  The residual block shares a similar architecture in Fig.~\ref{fig:dy-net} except with the identity shortcut mapping and without BN.
  We use 3$\times$3 kernels for all convolution filters.
  The number of feature maps in the convolution filters is set to 12 and 16 respectively for
  MNIST and SVHN.
  Meanwhile, the inception network consists of a convolution layer followed by an inception block (see Fig.~\ref{fig:inception}),
  another convolution layer, and two fully-connected layers.
  Regarding the hyper-parameters used in baselines, we
  select them from an appropriate search space detailed in Table~\ref{table:hyper}.
  We use the implementation in \url{https://github.com/Thrandis/EKFAC-pytorch} for EKFAC
  and implement our own EMSA in PyTorch since the official code released from \citet{li2017maximum} does not support GPU parallelization.

\endgroup

\textbf{Ablation study (Fig.~\ref{fig:abl-als})}
Each grid in Fig.~\ref{fig:abl-als} corresponds to a distinct combination of baseline and dataset. Its numerical value reports the performance difference between the following two training processes.
\begin{itemize}[noitemsep,topsep=0pt]
    \vskip -0.5in
    \item Accuracy of the baseline run with the best-tuned configuration which we report in Table~\ref{table:clf} and \ref{table:clf2}.
    \item Accuracy of DGNOpt with its parameter curvature set to the precondition matrix implied by the above best-tuned setup.

\end{itemize}
For instance, suppose the learning rate of EKFAC on MNIST is best-tuned to 0.01, then we simply set
$\Quutn \approx 0.01 \times \Qutn\Qutn^\transpose$ for all $t$.
From Theorem~\ref{thm:alg-connection}, these two training procedures only differ in the presence of $Q_{\theta\vx}^{t,n}$,
which allows EKFAC to adjust its update based on the change of $\vx_t\in\etaCL$.

\textbf{Runtime and memory complexity (Fig.~\ref{fig:runtime}).}
The numerical values are measured on the GTX 2070.

\textbf{Feedback analysis (Fig.~\ref{fig:large-lr}).}
We use the same inception-based network in Table~\ref{table:clf2}.

\textbf{Remark for EMSA (Footnote 3).}
Extended Method of Successive Approximations (EMSA) was originally proposed by \citet{li2017maximum} as an OCP-inspired method
for training \emph{feedforward} networks.
It considers the following minimization,
\begin{equation}
\begin{split}
    \theta_t^* =& \argmin {H}^\rho_t\left(\vz_t, \vz_{t+1}, \vp_{t}, \vp_{t+1}, \theta_t \right), \\
    \text{where }
    {H}^\rho_t\left(\vz_t, \vz_{t+1}, \vp_{t}, \vp_{t+1}, \theta_t \right)
  \coloneqq&
  H_t\left(\vz_t, \vp_{t+1}, \theta_t \right) +
  \frac{1}{2}\rho \norm{\vz_{t+1} - f_t(\vz_t,\theta_t)}_2 +
  \frac{1}{2}\rho \norm{\vp_t - \nabla_{\vz_t} H_t}_2
\end{split}  \label{eq:emsa}
\end{equation}
essentially augments the original Hamiltonian $H_t$ with the feasibility constraints on both forward states and backward co-states.
EMSA solves the minimization (\ref{eq:emsa})
with L-BFGS per layer at each training iteration.
In Table~\ref{table:clf} and \ref{table:clf2}, we extend their formula to accept $H_{t,n}$.
Due to the feasibility constraints,
the resulting modified Hamiltonian ${H}^\rho_{t,n}$ depends additionally on $\vx_{t+1}$ and $\vp_{t,n}$;
hence being different from the original Hamiltonian $H_{t,n}$.
As a result, the ablation analysis using Theorem~\ref{thm:alg-connection} is not applicable for EMSA.

\textbf{Cooperative training (Fig.~\ref{fig:multi-player}, Fig.~\ref{fig:players-comp}, and Table~\ref{table:multi-player}).}
The network consists of 4 convolutions followed by 2 fully-connected layers, and is activated by ReLU.
We use 3$\times$3 kernels with 32 feature maps for all convolutions and set the batch size to 128.

\textbf{Adaptive alignment with bandit (Fig.~\ref{fig:bandit} and Table~\ref{table:bandit}).}
We use the same ResNet18 as in classification for CIFAR10,
and a smaller residual network with 1 residual block for SVHN.
The residual block shares the same architecture as in Fig.~\ref{fig:dy-net} except without BN.
All convolution layers use 3$\times$3 kernels with 12 feature maps.
Again, the batch size is set to 128.
Note that in this experiment we use a slightly larger learning rate compared with the one used in Table 2. While DGNOpt achieves better final accuracies for both setups, in practice, the former tends to amplify the stabilization when we enlarge the information structure during training. Hence, it differentiates DGNOpt from other baselines.

Alg.~\ref{alg:dgnopt-adaptive} presents the pseudo-code of how DGNOpt can be integrated with any generic bandit-based algorithm (\markblue{marked as blue}).
For completeness, we also provide the pseudo-code of EXP3++ in Alg.~\ref{alg:exp3pp}.
We refer readers to \citet{seldin2014one} for the definition of $\xi_k(m)$ and $\eta_k$ (do not confuse with $\eta_{t,n}$ in the main context).

\begin{minipage}[t]{0.54\textwidth}
  \vskip -0.25in
  \begin{algorithm}[H]
     \caption{DGNOpt \markblue{with Multi-Armed Bandit (MAB)}} %
     \label{alg:dgnopt-adaptive}
  \begin{algorithmic}
     \STATE {\bfseries Input:} dataset $\mathcal{D}$, network $\{ f_i(\cdot, \theta_i) \}$, \markblue{number of alignments $M$}
     \STATE \markblue{Initialize the multi-armed bandit \texttt{MAB.init($M$)}}
     \REPEAT
     \STATE \markblue{Draw an alignment $m \leftarrow$ \texttt{MAB.sample()}.}
     \STATE \markblue{Construct $F\equiv \{F_t: t\in[T]\}$ according to $m$.}
     \STATE Compute $\vx_t$ by propagating $\vx_0 \sim \mathcal{D}$ through $F$
     \FOR{$t=T{-}1$ {\bfseries to} $0$ }{\hfill \markgreen{$\rhd$ Solve FNE or GR}}
         \STATE Solve the update $\delta \pi^*_{t,n}$ with (\ref{eq:fne-update}) or (\ref{eq:gr-update})
       \STATE Solve {${(V_\vx^{t,n}{,}V_{\vx\vx}^{t,n})}$} or {${(W_\vx^t{,}W_{\vx\vx}^t)}$} with (\ref{eq:fne-vxx}) or (\ref{eq:gr-vxx})
     \ENDFOR
     \STATE Set ${\vx}^\prime_0=\vx_0$
     \FOR{$t=0$ {\bfseries to} $T{-}1$}{\hfill \markgreen{$\rhd$ Update parameter}}
       \STATE Apply $\theta_{t,n} {\leftarrow} \theta_{t,n} {-} \delta \pi^*_{t,n}(\dvx_{t})$ with $\dvx_{t} {=} {\vx}^\prime_{t} {-} \vx_{t}$
       \STATE Compute ${\vx}^\prime_{t+1} = F_t({\vx}^\prime_t, \theta_{t,1}, \cdots, \theta_{t,N})$ %
     \ENDFOR
     \STATE \markblue{Compute the accuracy $r$ on validation set.}
     \STATE \markblue{Run \texttt{MAB.update($r$)}.}
     \UNTIL{ converges }
  \end{algorithmic}
  \end{algorithm}
\end{minipage}
\hfill
\begin{minipage}[t]{0.45\textwidth}
  \vskip -0.25in
  \begin{algorithm}[H]
     \caption{EXP3++ \citep{seldin2014one} } %
     \label{alg:exp3pp}
  \begin{algorithmic}
     \FUNCTION{\texttt{init($M$)}}
      \STATE $(k,M)\leftarrow(1,M)$
      \STATE $\forall m, L_k(m)=0$
     \ENDFUNCTION
     \STATE
     \FUNCTION{\texttt{sample()}}
      \STATE $\forall m, \epsilon_k(m) = \min\{ \frac{1}{2M},\frac{1}{2} \sqrt{\frac{\ln M}{kM}}, \xi_k(m) \} $
      \STATE $\forall m, \rho_k(m) = e^{-\eta_k {L}_k(m)}/\sum_{m^\prime} e^{-\eta_k {L}_k(m^\prime)}$
      \STATE $\forall m, \tilde{\rho}_k(m) = (1-\sum_{m^\prime} \epsilon_k(m^\prime)) \rho_k(m) + \epsilon_k(m)$
      \STATE Sample action according to $\tilde{\rho}_k(m)$
     \ENDFUNCTION
     \STATE
     \FUNCTION{\texttt{update($r_k^m$)}}
      \STATE $\ell_k^m = (1-{r_k^m})/{\tilde{\rho}_k(m)}$
      \STATE ${L}_{k+1}(m) = {L}_k(m) + \ell_k^m$
      \STATE $k \leftarrow k + 1$
     \ENDFUNCTION
  \end{algorithmic}
  \end{algorithm}
\end{minipage}

\textbf{Additional Experiments.}
\begin{minipage}[t]{0.6\textwidth}
  \begin{figure}[H]
    \vskip -0.15in
    \centering
    \includegraphics[height=2.65cm]{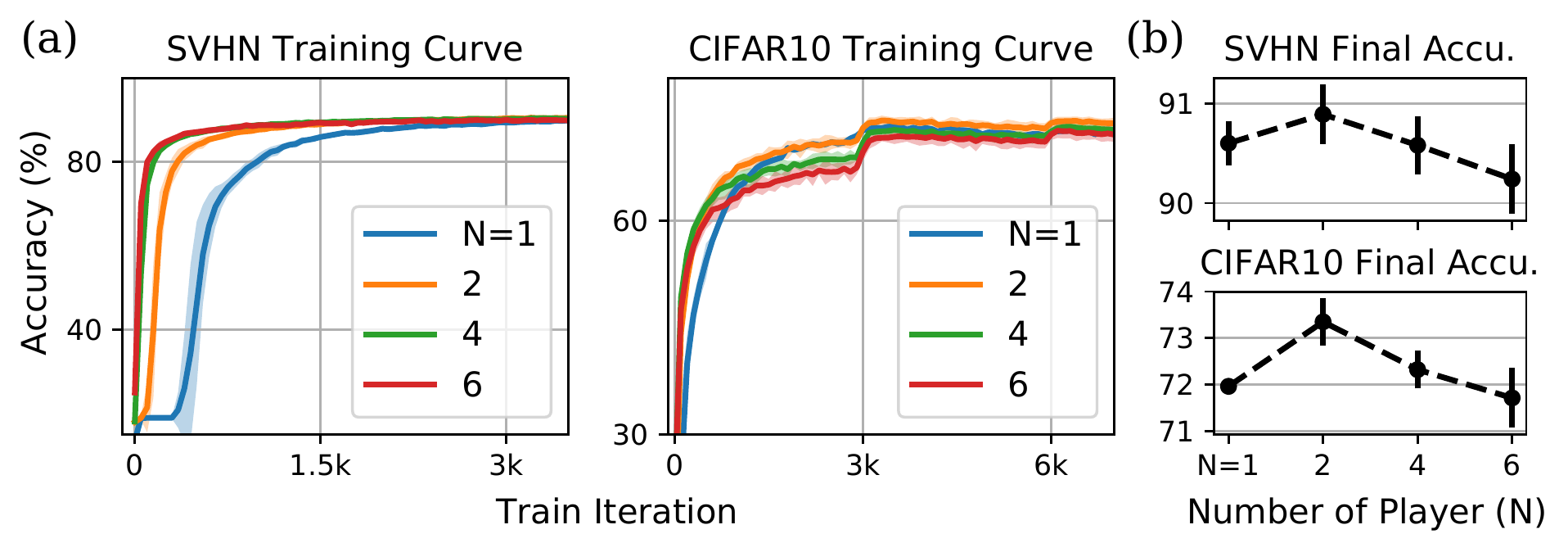}
    \vskip -0.15in
    \caption{
        (a) Training curve and (b) final accuracy as we vary the number of player ($N$)
        as a hyper-parameter of game-extended EKFAC. Similar to Fig.~\ref{fig:multi-player}, we also observe that $N{=}2$ gives the best final accuracy on both datasets.
    }
    \label{fig:11}
    \vskip -0.1in
  \end{figure}
\end{minipage}

\end{document}